\IfFileExists{iclr2027_conference.sty}{%
  \documentclass{article}%
  \usepackage{iclr2027_conference,times}%
}{%
  \documentclass[11pt]{article}%
}
\usepackage{amsmath,amssymb,mathtools,bm}
\usepackage{amsthm}
\usepackage{booktabs}
\usepackage[caption=false]{subfig} %
\usepackage{capt-of} %
\usepackage{algorithm}
\usepackage{algpseudocode}
\usepackage{xcolor}
\usepackage{pgfplots} %
\pgfplotsset{compat=1.18}
\usepackage{flafter,placeins} %
\usepackage{wrapfig} %
\usepackage{needspace} %
\usepackage{microtype}
\usepackage{xurl} %
\usepackage{hyperref}
\hypersetup{hidelinks} %
\usepackage[nameinlink,noabbrev]{cleveref}

\makeatletter
\@ifundefined{theHALG@line}{%
  \newcommand*{\theHALG@line}{\thealgorithm.\arabic{ALG@line}}%
}{%
  \renewcommand*{\theHALG@line}{\thealgorithm.\arabic{ALG@line}}%
}
\makeatother

\newtheorem{theorem}{Theorem}[section]

\newtheorem{proposition}[theorem]{Proposition}

\theoremstyle{definition}
\newtheorem{assumption}{Assumption}[section]
\theoremstyle{plain}
\crefname{assumption}{assumption}{assumptions}
\Crefname{assumption}{Assumption}{Assumptions}

\newcommand{\E}{\mathbb{E}}
\newcommand{\sg}{\operatorname{sg}}
\newcommand{\ours}{\textnormal{\textsc{DriftSpread}}}

\title{Drift, Then Spread: Learning One-Step Generative Policies for Online Reinforcement Learning}

\iclrfinalcopy

\author{%
\textbf{Chunyu Luo}$^{1,2}$ \quad
\textbf{Hongyu Zhao}$^{1,3}$ \quad
\textbf{Guangyi Ni}$^{4}$ \quad
\textbf{Yue Song}$^{5}$ \quad
\textbf{Lingkai Kong}$^{1}$ \\[3pt]
\makebox[\dimexpr\textwidth-2\tabcolsep\relax][c]{\normalfont\normalsize
$^{1}$The University of Hong Kong \quad $^{2}$Beihang University} \\
\makebox[\dimexpr\textwidth-2\tabcolsep\relax][c]{\normalfont\normalsize
$^{3}$Harbin Institute of Technology, Shenzhen \quad
$^{4}$Independent Researcher \quad $^{5}$Tsinghua University} \\[2pt]
\makebox[\dimexpr\textwidth-2\tabcolsep\relax][c]{\normalfont\normalsize
Correspondence to: \texttt{lingkaikong@hku.hk}.}
}

\begin{document}

\maketitle
\vspace{-12pt}
\lhead{Preprint}

\begin{abstract}
Generative policies offer expressive action distributions
for online reinforcement learning (RL), but many diffusion-
and flow-based policies require costly iterative sampling.
We propose \ours{}, a one-step generative policy based on
Drifting for online maximum-entropy (MaxEnt) RL.
To incorporate entropy without a tractable policy density,
we build on an energy-based view of Drifting and derive
critic-guided Drift and entropy-induced Spread from the
negative MaxEnt objective.
Our method learns these updates sequentially in the same
actor, using action regression for value improvement and
distribution matching to capture the effect of entropy.
The actor updates require neither direct samples from
the critic-defined Boltzmann target nor policy-density
or score estimation.
Under exact fitting with a fixed critic, we prove
convergence to the Boltzmann target in the small-step
limit followed by the long-time limit.
On MuJoCo, \ours{} achieves the highest aggregate
normalized return among all evaluated methods,
while outperforming the diffusion- and flow-based
baselines in both training and inference efficiency.
\end{abstract}

\section{Introduction}
\label{sec:introduction}

\noindent\begin{minipage}{\linewidth}
  \makeatletter
  \def\@captype{figure}
  \makeatother
  \centering
  \subfloat[Normalized return versus average training time.\label{fig:introduction-training-panel}]{%
    \includegraphics[width=4.8in]{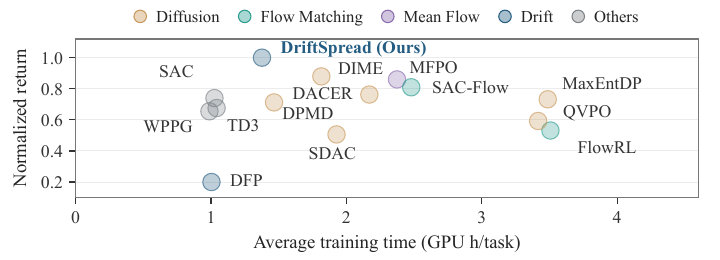}%
  }
  \par\smallskip
  \subfloat[Sampling steps and inference time.\label{fig:introduction-inference-panel}]{%
    \begin{minipage}{\linewidth}
      \centering
      \fontsize{8}{9.5}\selectfont
      \setlength{\tabcolsep}{0.8pt}
      \begin{tabular*}{\linewidth}{@{\extracolsep{\fill}}l*{14}{c}@{}}
        \toprule
        Method & SAC & TD3 & WPPG & \textbf{Ours} & DFP & FlowRL & DACER & MFPO &
        \shortstack{MaxEnt\\DP} & SDAC & DPMD & \shortstack{SAC-\\Flow} & DIME & QVPO \\
        \midrule
        Steps & 1 & 1 & 1 & \textbf{1} & 1 & 1 & 20 & 2 & 20 & 20 & 20 & 4 & 16 & 20 \\
        ms/action & 0.10 & 0.10 & 0.17 & \textbf{0.15} & 0.16 & 0.35 & 0.43 & 0.36 & 0.57 & 0.85 & 0.87 & 1.20 & 1.24 & 1.83 \\
        \bottomrule
      \end{tabular*}
    \end{minipage}%
  }
  \captionof{figure}{\textbf{Normalized return versus average training time (a),
and average inference time per action (b),
across five MuJoCo tasks.}}
  \label{fig:experiment-speed-performance}
\end{minipage}
\par\medskip

Generative policies offer expressive action distributions
for online reinforcement learning (RL).
Diffusion- and flow-based policies can capture multimodal
distributions that are difficult to represent with a single
Gaussian \citep{dong2025maximum,celik2025dime}.
However, many of these approaches generate actions through
iterative sampling, requiring repeated network evaluations
during training and inference.
Recent MeanFlow-based methods reduce this cost through
few-step or one-step action generation
\citep{dong2026mean,chen2025onestep,kim2026score}.

Drifting offers a different way to train one-step
generators \citep{deng2026drifting}.
At each training iteration, it uses target data to
move generated samples toward the desired distribution,
then trains the generator to reproduce the moved
samples from the same inputs.
Through these repeated updates, the generator learns
to produce samples from the target distribution in
a single forward pass at inference.
This approach has shown better performance than
MeanFlow variants in one-step image generation on
ImageNet \citep{deng2026drifting}.
These results motivate exploring Drifting for online
RL, where its use remains less studied.

We consider the maximum-entropy (MaxEnt) RL objective,
which balances expected return with policy entropy
to encourage exploration \citep{haarnoja2018soft}.
The critic defines a Boltzmann target distribution
for policy improvement, but does not directly provide
samples from it \citep{li2026reverse}.
We therefore adopt an energy-based view of Drifting
\citep{gretton2026wasserstein,han2026onestep},
using the negative MaxEnt policy-improvement objective
as the energy.
This energy guides updates to actions generated by
the current policy.
For the value term, critic gradients tell us how to
move actions toward higher values.
The entropy term is harder to use: an implicit actor
generates actions but does not provide a tractable
density for computing their distribution's entropy.
We therefore need training targets that capture
the effect of entropy without explicitly computing it.

We propose \ours{}, a one-step generative policy in the
Drifting model family for online MaxEnt RL.
Our key idea is to derive two updates from this energy
and fit them sequentially to the same actor through
repeated Drift--Spread cycles.
In each cycle, Drift uses critic gradients to move
generated actions toward higher values.
The actor then learns to reproduce these updated
actions from the same inputs through regression.
Spread captures the entropy contribution through
diffusion of the action distribution.
We approximate this diffusion by adding appropriately
scaled independent noise to actions from the updated
actor, then train the actor to match the resulting
distribution.
This allows the actor to learn the effect of entropy
without explicitly computing it.
For a fixed critic, we prove
that the learned policy converges to the Boltzmann
target distribution in the small-step, long-time limit.

Our main contributions are:
(1) We introduce \ours{}, a Drifting model for online
MaxEnt RL with expressive policies and one-step action
generation.
(2) We develop sequential Drift--Spread training through
critic-guided regression and  entropy regularization to
address the challenge of entropy computation for implicit
policies.
(3) On MuJoCo, \ours{} outperforms 13 baselines
in aggregate normalized return, with lower training
costs than the diffusion, flow, and MeanFlow baselines
(Figure~\ref{fig:experiment-speed-performance}).

\section{Related Work}
\label{sec:related-work}

\paragraph{Generative policies for online RL.}
Diffusion and flow models provide expressive action
distributions for online RL.
One line of work optimizes policies by backpropagating
critic gradients through diffusion or flow sampling
\citep{wang2024diffusion,lv2025flow}.
DIME also differentiates through diffusion sampling
while optimizing a variational lower bound on the
MaxEnt objective \citep{celik2025dime}.

Another line of work constructs training targets
for learning critic-defined Boltzmann distributions.
DQS uses weighted averages of critic gradients
to construct score targets \citep{jain2025sampling}.
SDAC and MaxEntDP instead use critic values to weight
denoising supervision and noise estimates, respectively
\citep{ma2025efficient,dong2025maximum}.
Reverse flow matching unifies these gradient-expectation
and noise-expectation formulations through posterior-mean
estimation
\citep{li2026reverse}.

 These approaches use iterative action sampling,
motivating few-step and one-step policies for online RL.
MeanFlow was introduced for one-step generative modeling
and learns average velocities \citep{geng2025mean}.
Recent online RL methods use this formulation
to learn policies with fewer generation steps.
FPMD trains rectified-flow and MeanFlow policies
through policy mirror descent, with both variants
supporting one-step inference \citep{chen2025onestep}.
MFPO trains few-step MeanFlow policies using velocity
estimates for critic-defined Boltzmann targets
\citep{dong2026mean}, while SOM constructs velocity
targets from critic-based score estimates to train
one-step MeanFlow policies \citep{kim2026score}.

\paragraph{Drifting models and policy learning.}
Drifting learns one-step generators by repeatedly
moving generated samples and training the generator
to reproduce their updated positions.
This approach has achieved state-of-the-art performance
in one-step image generation on ImageNet
\citep{deng2026drifting}.
Subsequent work relates specific Drifting fields
to Wasserstein gradient flows
\citep{gretton2026wasserstein} and uses
optimal-transport couplings to construct sample
updates \citep{he2026sinkhorn}.
W-Flow derives sample movements from an objective
over distributions and trains a one-step generator
on the resulting samples \citep{han2026onestep}.

Drifting has also been applied to policy learning.
For offline RL, DriftQL combines critic-guided
policy improvement with a Drifting regularizer
based on the behavior distribution
\citep{houssaini2026drift}.
For offline-to-online RL, Drifting Field Policy
derives a KL-regularized policy update through
Wasserstein gradient flow and constructs training
targets from actions ranked by the critic
\citep{koo2026drifting}.
Our work focuses on online MaxEnt RL from scratch.

\section{Preliminaries}
\label{sec:preliminaries}

\subsection{Maximum-entropy RL and soft policy iteration}
\label{sec:prelim-rl}

We consider a discounted Markov decision process (MDP) with
state space $\mathcal S$, continuous action space
$\mathcal A\subseteq\mathbb R^d$, transition kernel
$P(\cdot\mid s,a)$, reward $r(s,a)$, and discount factor
$\gamma\in[0,1)$.
A policy $\pi(a\mid s)$ specifies a density over actions
at each state.
Maximum-entropy (MaxEnt) RL learns a policy that maximizes
discounted rewards together with policy entropy:
\begin{equation}
    J(\pi)
    =\E_\pi\left[
        \sum_{t=0}^{\infty}\gamma^t
        \bigl(r(s_t,a_t)
        +\lambda\mathcal H(\pi(\cdot\mid s_t))\bigr)
    \right],
    \label{eq:maxent-objective}
\end{equation}
where
$\mathcal H(\pi(\cdot\mid s))
=-\E_{a\sim\pi(\cdot\mid s)}[\log\pi(a\mid s)]$
is the conditional differential entropy, and $\lambda>0$
controls the tradeoff between reward and action diversity.
The expectation is over trajectories generated by $\pi$
from a fixed initial-state distribution.

Soft policy iteration alternates between policy evaluation
and policy improvement \citep{haarnoja2018soft}.
Evaluation holds $\pi$ fixed and computes its soft
action-value function $Q^\pi$, which accounts for rewards
and future entropy bonuses.
This function satisfies $Q^\pi=\mathcal T^\pi Q^\pi$,
where the soft Bellman operator is
\begin{equation}
    (\mathcal T^\pi Q)(s,a)
    =r(s,a)+\gamma\E_{s'}\left[
        \E_{a'\sim\pi(\cdot\mid s')}[Q(s',a')]
        +\lambda\mathcal H(\pi(\cdot\mid s'))
    \right],
    \label{eq:soft-bellman-operator}
\end{equation}
where $s'\sim P(\cdot\mid s,a)$.

Improvement then holds the critic $Q$ fixed and updates
the policy at each state to maximize
$\E_{a\sim\pi(\cdot\mid s)}[Q(s,a)]
+\lambda\mathcal H(\pi(\cdot\mid s))$.
When optimizing over all action densities, the maximizing
policy has the Boltzmann form
\begin{equation}
    \pi_Q^*(a\mid s)
    =\frac{\exp(Q(s,a)/\lambda)}{Z_Q(s)},
    \qquad
    Z_Q(s)=\int_{\mathcal A}\exp(Q(s,a)/\lambda)\,da,
    \label{eq:boltzmann-policy}
\end{equation}
assuming $0<Z_Q(s)<\infty$.
Thus, each fixed critic defines a target action distribution
for policy improvement.
The practical task is to learn a policy that approximates
this distribution and generates actions efficiently.

\begin{figure}[!t]
    \centering
    \subfloat[\label{fig:framework-training}]{%
        \includegraphics[width=0.49\linewidth]{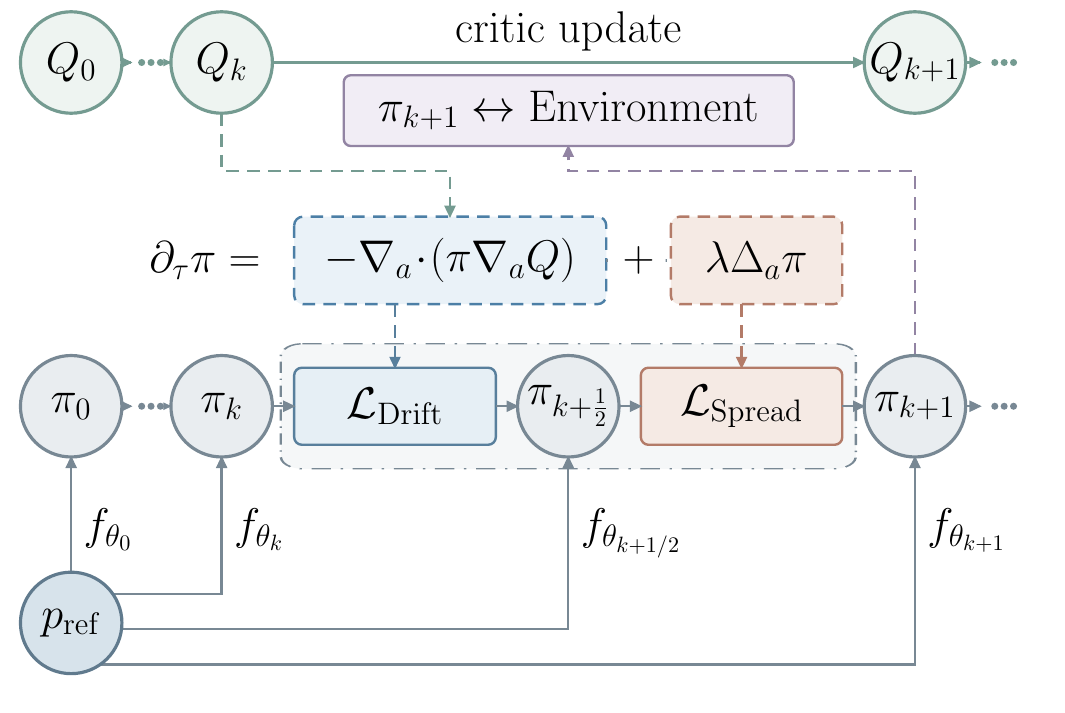}%
    }\hfill
    \subfloat[\label{fig:framework-energy}]{%
        \includegraphics[width=0.49\linewidth]{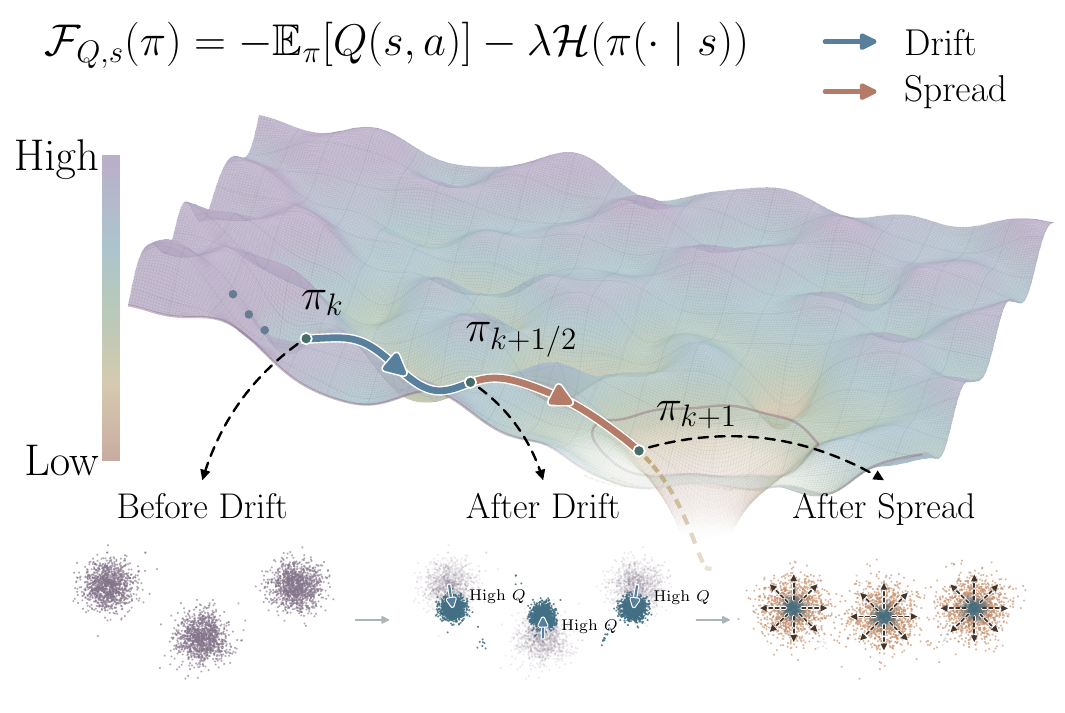}%
    }
\caption{
\textbf{Overview of \ours{}.}
(a) Training alternates between critic learning and
sequential Drift--Spread actor updates, using action
regression for Drift and distribution matching for Spread.
The maps $f_\theta$ represent the same actor at
different training stages, each generating an
action sample from reference noise in one forward pass.
(b) The surface schematically represents the negative
MaxEnt objective over policy distributions.
The samples below illustrate how Drift moves actions
toward higher critic values and Spread then broadens
the post-Drift distribution to capture the effect
of entropy.
}
\label{fig:framework}
\end{figure}

\subsection{Drifting models and Wasserstein gradient flows}
\label{sec:prelim-drifting}

Drifting trains a one-step generator by repeatedly moving its
samples and fitting the generator to their new positions
\citep{deng2026drifting}.
The generated distribution evolves during training, while
sampling requires only one forward pass.

The original Drifting field combines attraction toward data
samples with repulsion from generated samples.
Energy-based extensions start from an objective $\mathcal F(q)$
over the generated distribution $q$.
For example, $\mathcal F(q)$ can measure the difference between
$q$ and a target distribution.
They use a Wasserstein gradient flow to derive sample movements
that reduce this objective
\citep{gretton2026wasserstein,han2026onestep}.
This flow describes continuous-time steepest descent of
$\mathcal F$ under the Wasserstein distance
\citep{jordan1998variational,ambrosio2008gradient}.
The first variation $\delta\mathcal F/\delta q$ describes how
$\mathcal F$ responds to small changes in $q$;
its negative spatial gradient gives the sample velocity.
Under suitable regularity and decay conditions,
\begin{equation}
    V_q(x)
    =-\nabla_x\frac{\delta\mathcal F}{\delta q}(q)(x),
    \qquad
    \partial_\tau q_\tau
    +\nabla_x\cdot(q_\tau V_{q_\tau})=0.
    \label{eq:energy-velocity}
\end{equation}
Here, $\tau$ denotes continuous flow time.
The first equation specifies how individual samples move,
and the second describes how their density changes
as a result.

To train a generator with this velocity,
\citet{han2026onestep} use a forward Euler step to construct
targets for fitting.
At iteration $k$, let $q_k$ denote the distribution of
$x=f_{\theta_k}(z)$, where $z\sim p_{\mathrm{ref}}$.
When the velocity can be evaluated at $q_k$, setting
$V_k=V_{q_k}$ with a step size $\eta>0$ gives
\begin{equation}
    \widetilde x=\sg[x+\eta V_k(x)],
    \qquad
    f_{\theta_{k+1}}(z)\approx\widetilde x.
    \label{eq:drifting-target}
\end{equation}
The generator is fitted using the same input $z$, while $\sg$
keeps the target fixed.
The generator thus learns to reproduce the moved samples
in a single forward pass.

\section{Method}
\label{sec:method}

\ours{} learns a one-step generative policy through
two sequential updates to the same actor:
critic-guided Drift and entropy-induced Spread.
We first derive these updates from the MaxEnt objective
(\cref{sec:method-objective}), then show how to train
the actor using action regression for Drift and
distribution matching for Spread
(\cref{sec:method-training}).
We next analyze fixed-critic convergence under exact fitting
(\cref{sec:method-guarantees}), and finally describe
the practical training procedure
(\cref{sec:method-practical}).
Figure~\ref{fig:framework} illustrates the overall method.

\subsection{Decomposing MaxEnt policy improvement}
\label{sec:method-objective}
\label{sec:method-splitting}

The Drifting framework updates a generator by fitting
its outputs to targets displaced by a velocity field.
For MaxEnt policy improvement, we derive this field
from the negative MaxEnt objective.
For a fixed state $s$, critic $Q$, and temperature
$\lambda>0$, define the distributional energy
\[
    \mathcal F_{Q,s}(\pi)
    =-\E_\pi[Q(s,a)]
     -\lambda\mathcal H(\pi(\cdot\mid s)).
\]
For smooth positive densities with sufficient decay
and finite second moments, the energy--velocity
relation in \cref{eq:energy-velocity} gives
\[
    V_\pi(a\mid s)
    =-\nabla_a\frac{\delta\mathcal F_{Q,s}}{\delta\pi}(a)
    =\nabla_aQ(s,a)-\lambda\nabla_a\log\pi(a\mid s).
\]
The critic gradient is directly available, but the
entropy contribution depends on the policy score
$\nabla_a\log\pi$, which an implicit actor does not
provide directly.

To implement the entropy contribution without
estimating this score, we consider the corresponding
evolution of the policy distribution.
Substituting $V_\pi$ into the continuity equation gives
\begin{equation}
    \partial_\tau\pi
    =\underbrace{-\nabla_a\cdot(\pi\nabla_aQ)}
       _{\text{critic-guided drift}}
     +\underbrace{\lambda\Delta_a\pi}
       _{\text{entropy-induced spreading}}.
    \label{eq:method-fokker-planck}
\end{equation}
The drift term transports actions toward higher
critic values. We approximate this transport by
an explicit step along the critic gradient.

The spreading term describes diffusion.
When this diffusion acts alone, it preserves the
mean and increases the variance of each action
coordinate by $2\lambda\eta$ over a time interval
$\eta$.
We approximate its effect over a small time step
by adding independent noise with the corresponding
scale:
\[
    \epsilon_\eta=\sqrt{2\lambda\eta}\,\xi,
    \qquad
    \E[\xi]=0,
    \qquad
    \E[\xi\xi^\top]=I_d.
\]
Here, $\xi$ follows a fixed noise distribution,
such as Gaussian, Gamma, or uniform noise,
standardized to satisfy the moment conditions above.

Let $K_{\lambda,\eta}$ denote the distribution of
$\epsilon_\eta$.
Starting from $a\sim\pi_k(\cdot\mid s)$, an ideal
cycle applies Drift followed by Spread:
\[
    a
    \xrightarrow{\mathrm{Drift}}
    a+\eta\nabla_a Q(s,a)
    \xrightarrow{\mathrm{Spread}}
    a+\eta\nabla_a Q(s,a)+\epsilon_\eta,
\]
where $\epsilon_\eta\sim K_{\lambda,\eta}$ is
independent of $a$.

\subsection{Learning Drift and Spread sequentially}
\label{sec:method-drift}
\label{sec:method-training}

As illustrated in Figure~\ref{fig:framework},
we update the same actor in two stages, using action
regression for Drift and distribution matching for Spread: $\theta_k
    \xrightarrow{\mathrm{Drift}}\theta_{k+1/2}
    \xrightarrow{\mathrm{Spread}}\theta_{k+1}.$
The critic remains fixed throughout each cycle.

\paragraph{Learning Drift from the critic.}
Following Drifting \citep{deng2026drifting}, we turn
the critic-gradient step into a regression target.
For each input $(s,z)$, we displace the current
actor's output along the critic gradient and fit
the resulting target using the same input:
\begin{align}
    \widetilde a_{k+1/2}(s,z)
    &=\sg\!\left[
        f_{\theta_k}(s,z)
        +\eta\nabla_a Q\bigl(s,f_{\theta_k}(s,z)\bigr)
    \right],
    \label{eq:method-drift-target}\\
    \mathcal L_{\rm Drift}(\theta;\theta_k)
    &=\tfrac12\E_{s,z}
    \left\|
        f_\theta(s,z)-\widetilde a_{k+1/2}(s,z)
    \right\|^2.
    \label{eq:method-drift-loss}
\end{align}
Here, states $s$ are sampled from the replay buffer
and $z\sim p_{\mathrm{ref}}$.
The stop-gradient operation $\sg$ holds the target
fixed during fitting.
Starting from $\theta_k$, this fit produces the
post-Drift actor $f_{\theta_{k+1/2}}$.

\paragraph{Learning Spread through distribution matching.}
Spread trains the actor to represent the distribution of
$a+\epsilon_\eta$, where
$a=\sg[f_{\theta_{k+1/2}}(s,z)]$.
Let $\widetilde\pi_{k+1}(\cdot\mid s)$ denote this target
distribution.
Because the noise is independent and has zero mean,
paired regression using the same input $z$ gives
\[
    \E_{z,\epsilon_\eta}
    \|f_\theta(s,z)-a-\epsilon_\eta\|^2
    =
    \E_z\|f_\theta(s,z)-a\|^2
    +\E\|\epsilon_\eta\|^2.
\]
The noise contributes only a constant, so the post-Drift
actor already minimizes this loss.
To learn the spreading effect, we instead match the
actor and target distributions.

Our actor generates samples but does not provide a
tractable action density, making direct KL minimization
difficult.
We therefore use squared maximum mean discrepancy (MMD),
which can be estimated from samples without evaluating
their densities \citep{li2015generative}:
\begin{equation}
    \mathcal L_{\mathrm{Spread}}
    (\theta;\theta_{k+1/2})
    =
    \E_s\operatorname{MMD}_{h}^{2}
    \bigl(
        \pi_\theta(\cdot\mid s),
        \widetilde\pi_{k+1}(\cdot\mid s)
    \bigr),
    \label{eq:method-mmd-population}
\end{equation}
where $h$ is a positive-definite kernel on the action
space.
The target distribution is held fixed during this update.

We choose a kernel with a frequency representation
and compute this objective using its equivalent
expression in terms of characteristic functions (CFs)
\citep{ansari2020characteristic}.
Define the actor's CF as
$\varphi_\theta(\omega\mid s)
:=\E_z[e^{i\omega^\top f_\theta(s,z)}]$,
where $\omega\in\mathbb R^d$.
Since the spreading noise is independent of the
post-Drift action, the target CF is
\begin{equation}
    \widetilde\varphi_{k+1}(\omega\mid s)
    =
    \widehat K_{\lambda,\eta}(\omega)
    \varphi_{\theta_{k+1/2}}(\omega\mid s),
    \label{eq:method-cf-target}
\end{equation}
where $\widehat K_{\lambda,\eta}$ is the known noise CF.
We estimate the post-Drift actor's CF from its samples
and multiply it by this factor.
For Gaussian spreading noise
$K_{\lambda,\eta}=\mathcal N(0,2\lambda\eta I_d)$,
$\widehat K_{\lambda,\eta}(\omega)
=e^{-\lambda\eta\|\omega\|^2}$.
Appendix~\ref{sec:appendix-cf-targets}
gives the CFs for the other noise choices.

Starting from $\theta_{k+1/2}$, we obtain
$\theta_{k+1}$ by minimizing the following objective
with the target CF held fixed:
\begin{equation}
    \mathcal L_{\mathrm{Spread}}(\theta;\theta_{k+1/2})
    =
    \E_s\E_{\omega\sim\rho}
    \left|
        \varphi_\theta(\omega\mid s)
        -
        \sg\bigl[
            \widetilde\varphi_{k+1}(\omega\mid s)
        \bigr]
    \right|^2.
    \label{eq:method-cf-population}
\end{equation}
We estimate this loss using states sampled from the
replay buffer, generated actions, and frequencies
drawn from $\rho$.
The actor and target CFs are evaluated at the same
frequencies.

Here, $\rho$ determines the kernel through
$h(x,y)=\E_{\omega\sim\rho}
[\cos(\omega^\top(x-y))]$.
When $\rho=\mathcal N(0,\ell^{-2}I_d)$,
this objective equals squared MMD with the Gaussian
kernel
$h(x,y)=\exp(-\|x-y\|^2/(2\ell^2))$.
We use this choice in our experiments.

\begin{wrapfigure}[16]{r}{0.49\linewidth}
  \centering
  \includegraphics[width=\linewidth]
    {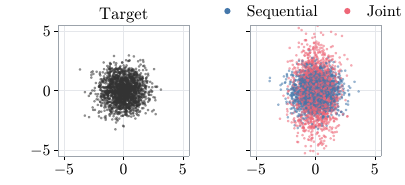}
  \par
  \caption{\textbf{Sequential versus joint fitting for a
  Gaussian Boltzmann target.}
  Left: target samples.
  Right: samples from the fitted policies.
  Joint's loss weight is tuned to minimize KL
  divergence to the known target.}
  \label{fig:method-sequential-joint}
\end{wrapfigure}
\noindent\textbf{Why sequential fitting?}\quad
Sequential fitting first minimizes the Drift
regression loss $\mathcal L_{\mathrm{Drift}}$, then
constructs and fits the Spread target using the
updated actor.
A joint alternative constructs both targets from
the actor before Drift and minimizes
$\mathcal L_{\mathrm{Drift}}
+\beta\mathcal L_{\mathrm{Spread}}$,
where $\beta>0$ weights the Spread MMD loss.
Because regression and MMD penalize changes in the
policy differently, a single weight need not recover
the intended balance between Drift and Spread
across action coordinates.
Figure~\ref{fig:method-sequential-joint} illustrates
this issue in a two-dimensional Gaussian example.
After repeated fitting, Sequential closely matches
the Boltzmann target, whereas Joint remains too
narrow in one coordinate and too wide in the other,
despite tuning its weight.
Appendix~\ref{sec:appendix-design-analysis}
analyzes this bias.

\subsection{Theoretical analysis}
\label{sec:method-guarantees}

We next show that sequential fitting can recover
the Boltzmann target for a fixed critic.
Under exact fitting, Drift reproduces the
critic-guided action update, and Spread matches
the distribution obtained by adding noise to the
post-Drift actions.
Each cycle therefore implements the distribution
update derived in \cref{sec:method-objective}.

Fix a state $s$, critic $Q$, and temperature $\lambda>0$.
Let $\pi_k^{(\eta)}$ denote the actor's policy after
$k$ cycles with step size $\eta$, starting from
the same initial policy $\pi_0$ for every $\eta$.
To compare different step sizes at a common
evolution time $\tau$, we take
$k=\lfloor\tau/\eta\rfloor$.

\begin{theorem}[Boltzmann convergence]
\label{thm:main-fixed-critic-convergence}
Under the normalizability assumption
$0<Z_Q(s)<\infty$, suppose that
(i) the population Drift regression loss
$\mathcal L_{\mathrm{Drift}}$ and Spread MMD loss
$\mathcal L_{\mathrm{Spread}}$ at state $s$
are zero after their respective updates in every cycle,
{and the MMD kernel $h$ is characteristic,
meaning that zero MMD implies identical distributions};
(ii) $Q(s,\cdot)$ is twice continuously differentiable
and satisfies
$\langle a,\nabla_a Q(s,a)\rangle
\le C(1+\|a\|^2)$
for all $a$ and some finite $C\ge0$, which prevents the continuous dynamics from escaping to infinity in finite time; and
(iii) the Spread noise is
$\epsilon_\eta=\sqrt{2\lambda\eta}\,\xi$,
where $\xi$ has an $\eta$-independent distribution
with $\E[\xi]=0$, $\E[\xi\xi^\top]=I_d$,
and $\E\|\xi\|^3<\infty$,
with draws independent across cycles and independent
of the current action.
Then the fitted policies satisfy
\begin{equation}
    \lim_{\tau\to\infty}\lim_{\eta\downarrow0}
    \pi_{\lfloor\tau/\eta\rfloor}^{(\eta)}
    =\pi_Q^*
    \qquad\text{weakly}.
    \label{eq:method-fixed-critic-limit}
\end{equation}
\end{theorem}

The two limits are taken in order:
first, taking $\eta\downarrow0$ at fixed $\tau$
recovers the MaxEnt diffusion in
\cref{eq:method-fokker-planck};
then, letting $\tau\to\infty$ brings this diffusion
to the Boltzmann target.
The proof follows from
\cref{thm:cf-target-convergence} in the appendix.

\subsection{Implementation Details}
\label{sec:method-practical}
\label{sec:method-value-evaluation}

We implement \ours{} as an off-policy actor--critic
algorithm that reuses past transitions stored in
a replay buffer.
Training combines data collection, critic learning,
and sequential Drift--Spread actor updates.

\noindent\textbf{Data collection.}\quad
At each environment step, we sample an action
$a\sim\pi_\theta(\cdot\mid s)$ from the current actor
and execute it in the environment.
We store the resulting transition
$(s,a,r,s',\delta)$ in the replay buffer for
subsequent critic and actor updates.

\noindent\textbf{Critic training.}\quad
Following prior generative-policy methods
\citep{jain2025sampling,li2026reverse},
we use standard critic targets without an entropy
term, while retaining entropy regularization in
the actor updates.
We maintain two critics and use the smaller of
their target values to reduce overestimation
\citep{fujimoto2018addressing}.
For a replay transition $(s,a,r,s',\delta)$,
both critics minimize the squared error to the
detached target
$y=r+\gamma(1-\delta)\min_{i=1,2}\bar Q_i(s',a')$,
where $a'\sim\pi_{\bar\theta}(\cdot\mid s')$.
Here, $\bar Q_i$ and $\pi_{\bar\theta}$ denote
the target critics and target policy, respectively,
and $\delta$ indicates termination.

\noindent\textbf{Actor training.}\quad
After each critic update, we hold the critics fixed
and perform one Drift--Spread cycle using states
sampled from the replay buffer.
For these states, the current actor generates the
action samples used to construct the Drift targets.
We take one optimizer step on the Drift regression
loss $\mathcal L_{\mathrm{Drift}}$.
We then construct the Spread target from the updated
actor and take one optimizer step on the Spread
MMD loss $\mathcal L_{\mathrm{Spread}}$.

The full training procedure is provided in
Appendix~\ref{sec:method-critics}.

\FloatBarrier
\section{Experiments}
\label{sec:experiments}

We first use a toy experiment to examine the role
of Spread in learning a Boltzmann target distribution.
We then compare control performance and computational
efficiency on MuJoCo.
We also investigate how the choice of spreading noise
and sequential versus joint fitting affect policy learning.

\pgfplotsset{compat=1.18}

\providecommand{\RingCompactFigure}[1][.52\linewidth]{%
  \includegraphics[width=#1]{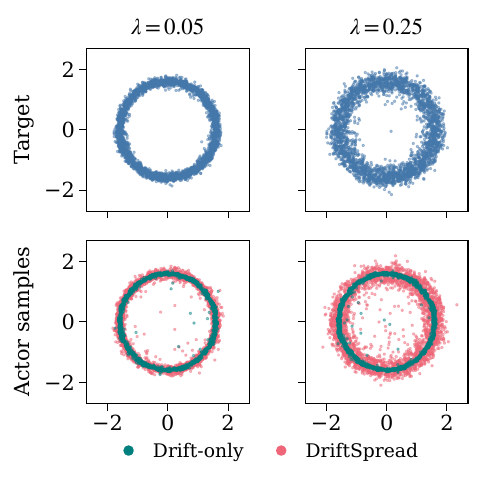}%
}
\providecommand{\RingSweepFigure}[1][\linewidth]{%
  \includegraphics[width=#1]{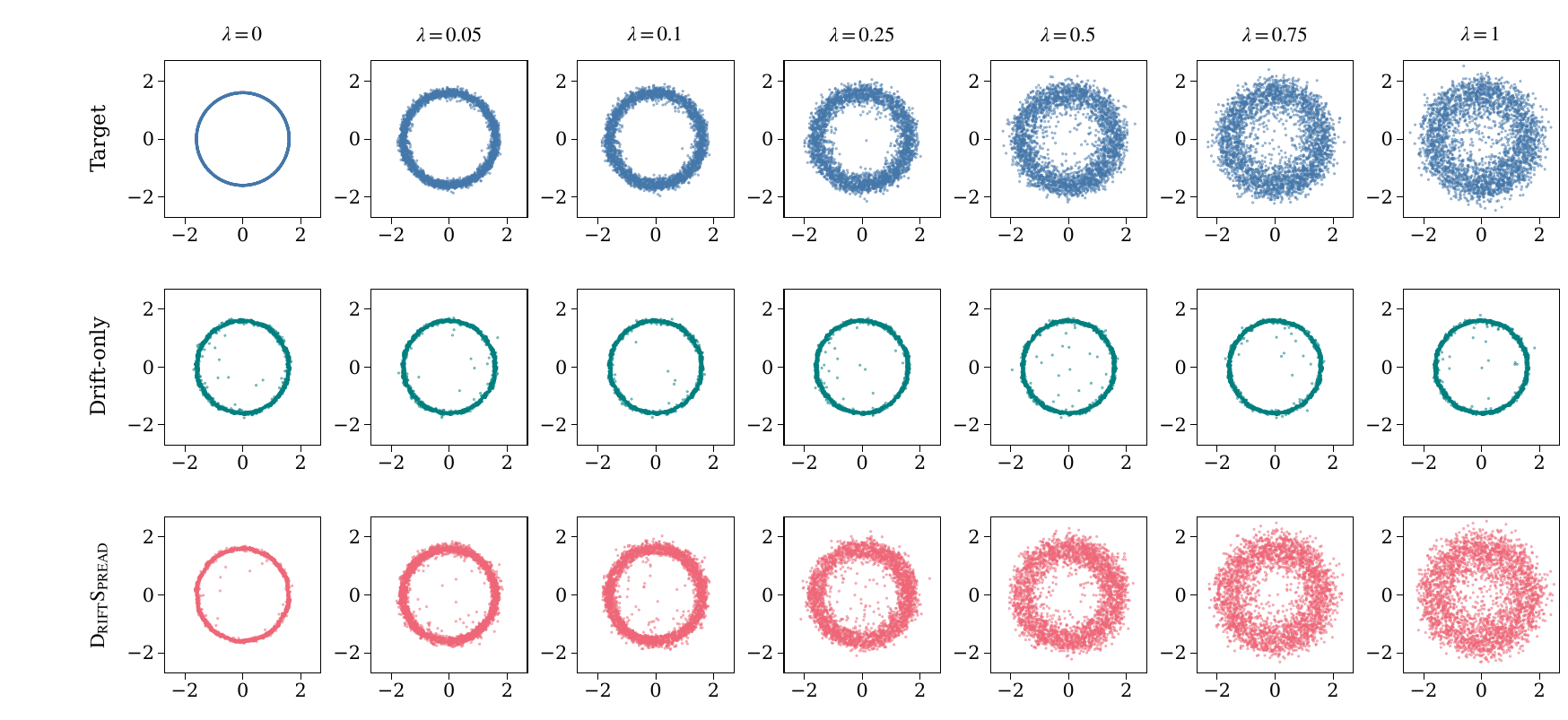}%
}

\vspace{-8pt}
\subsection{Toy Experiment}
\label{sec:experiment-cfm-toy}

\begin{wrapfigure}[15]{l}{0.38\textwidth}
    \centering
    \setlength{\parskip}{0pt}
    \setlength{\abovecaptionskip}{-7pt}
    \vspace{-2.4em}
    \RingCompactFigure[\linewidth]\par
    \caption{
    \textbf{Learning a ring-shaped Boltzmann distribution.}
    }
    \label{fig:experiment-cfm-toy}
\end{wrapfigure}
We first evaluate \ours{} on a two-dimensional
Boltzmann target
$\pi^*_{Q,\lambda}(a)\propto\exp(Q(a)/\lambda)$,
with $Q(a)=-(\|a\|_2^2-r_0^2)^2/(4\sigma_Q^2)$.
The target concentrates around a ring of radius
$r_0$, with a width that increases with temperature
$\lambda$.
To assess the role of Spread, we compare \ours{}
with Drift-only using the same MLP actor architecture.
Neither method uses samples from the target
distribution during training.

Figure~\ref{fig:experiment-cfm-toy} shows target
samples in blue in the top row, and samples from
Drift-only (teal) and \ours{} (red) in the bottom row.
At the higher temperature, the blue target samples
spread more widely around the ring.
The teal samples from Drift-only remain tightly
concentrated near the ring at both temperatures.
By contrast, the red samples from \ours{} capture
the wider spread of the higher-temperature target,
showing the effect of Spread on the learned distribution.

\par
\ifnum\value{WF@wrappedlines}>1
  \vspace{\dimexpr\value{WF@wrappedlines}\baselineskip\relax}
\fi
\WFclear

\begin{figure}[!t]
  \centering
  \includegraphics[width=0.95\linewidth]{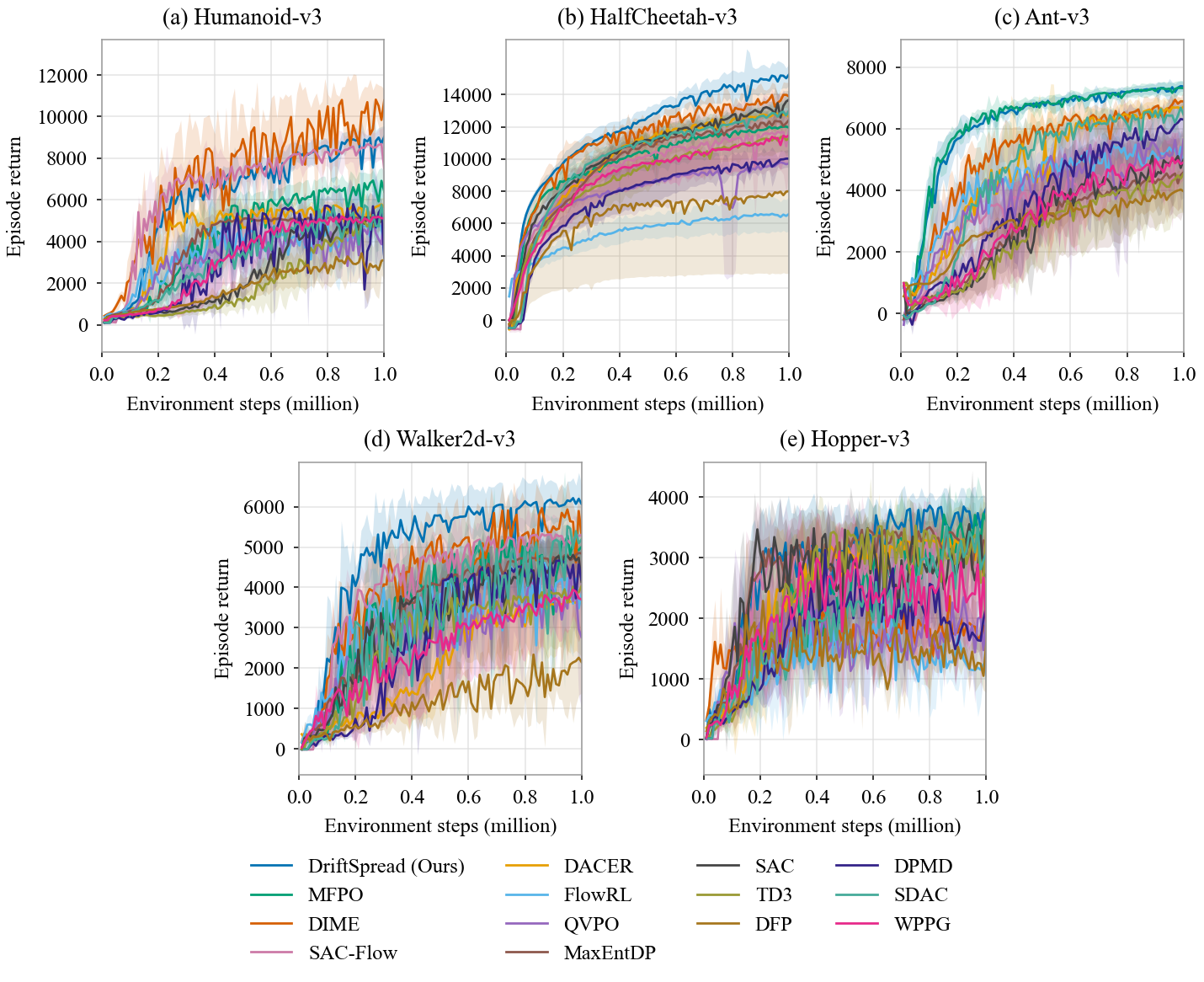}
  \caption{\textbf{Comparison with 13 baselines on five MuJoCo-v3 locomotion tasks.}
  Curves show five-seed mean evaluation returns with sample standard deviations
  shaded and no smoothing.}
  \label{fig:experiment-mujoco}
\end{figure}

\subsection{MuJoCo Benchmarks}
\label{sec:experiment-setup}
\label{sec:experiment-comparison}
\label{sec:experiment-training-cost}

\paragraph{Experimental setup.}
We evaluate the control performance and computational
efficiency of \ours{} on five MuJoCo-v3 tasks:
HalfCheetah, Hopper, Walker2d, Ant, and Humanoid.
For each task, every method is trained for one million
environment interactions using five random seeds.

We use SAC \citep{haarnoja2018soft}
and TD3 \citep{fujimoto2018addressing}
as standard actor--critic baselines.
To compare with more expressive generative policies,
we include diffusion-based methods:
DIME \citep{celik2025dime},
DACER \citep{wang2024diffusion},
QVPO \citep{ding2024diffusion},
MaxEntDP \citep{dong2025maximum},
and DPMD and SDAC \citep{ma2025efficient},
as well as the flow-based methods
FlowRL \citep{lv2025flow}
and SAC-Flow \citep{zhang2026sacflow}.
To compare with policies designed for efficient
action generation, we include MeanFlow-based
MFPO \citep{dong2026mean}, which generates actions
in a few steps.
We further include DFP \citep{koo2026drifting}
and WPPG \citep{zhu2026wasserstein},
two recent one-step policies based on
Wasserstein gradient flows.

To keep these comparisons consistent, we align
batch size, network capacity, and
update-to-data ratio across methods.
We implement \ours{} in JAX and use the same
framework for all baselines.\footnote{For baselines
originally implemented in PyTorch, we port their
official implementations to JAX for consistent
runtime comparisons.}
Training and inference time are measured on
the same NVIDIA RTX 4090 GPU.
Full configurations are
provided in
Appendix~\ref{sec:appendix-experimental-details}.

\paragraph{Performance and efficiency.}
Figure~\ref{fig:experiment-speed-performance}
shows that \ours{} achieves the highest aggregate
normalized return among the compared methods,
with lower training costs than the diffusion,
flow, and MeanFlow baselines.
It also reduces average inference time to
$0.15$\,ms per action, compared with $1.24$\,ms
for DIME and $0.36$\,ms for the two-step MeanFlow
policy MFPO.
These comparisons show improvements in both
control performance and computational efficiency.

Among one-step generative policies, \ours{} achieves
higher aggregate returns than DFP and WPPG with
similar inference times, although these baselines
require slightly less training time.
DFP was originally designed for offline-to-online
fine-tuning \citep{koo2026drifting}, whereas our
experiments evaluate online learning from scratch.
WPPG fits noisy action targets through paired
regression using the same generator inputs
\citep{zhu2026wasserstein}.
This regression averages out independent zero-mean
target noise in expectation.
Our Spread update instead learns this distribution
through distribution matching.
This difference may contribute to the observed
performance gap.

The per-task learning curves
(Figure~\ref{fig:experiment-mujoco}) complement
these aggregate results.
\ours{} achieves the highest late-training mean
returns on HalfCheetah-v3 and Walker2d-v3,
while remaining competitive with the strongest
baselines on the other tasks.

\paragraph{Impact of Spreading Noise.}
We examine how the choice of noise in Spread affects
policy learning.
We compare Gaussian noise with the Gamma-mixture,
Resolvent, and Uniform alternatives defined in
Appendix~\ref{sec:appendix-cf-targets}.
All four choices have zero mean and the same covariance.
We keep the actor architecture and training procedure unchanged within each task.

Figure~\ref{fig:experiment-cf-target-ablation} shows
that Gaussian noise achieves the highest or near-highest
late-training mean returns across the three tasks.
On Humanoid-v3 and HalfCheetah-v3, Gaussian reaches
higher returns than Gamma and Uniform, while
Resolvent remains closer to Gaussian.
The differences are smaller on Ant-v3.
These results support Gaussian noise as our default choice.

\begin{figure}[!htbp]
  \centering
  \includegraphics[width=\linewidth]
    {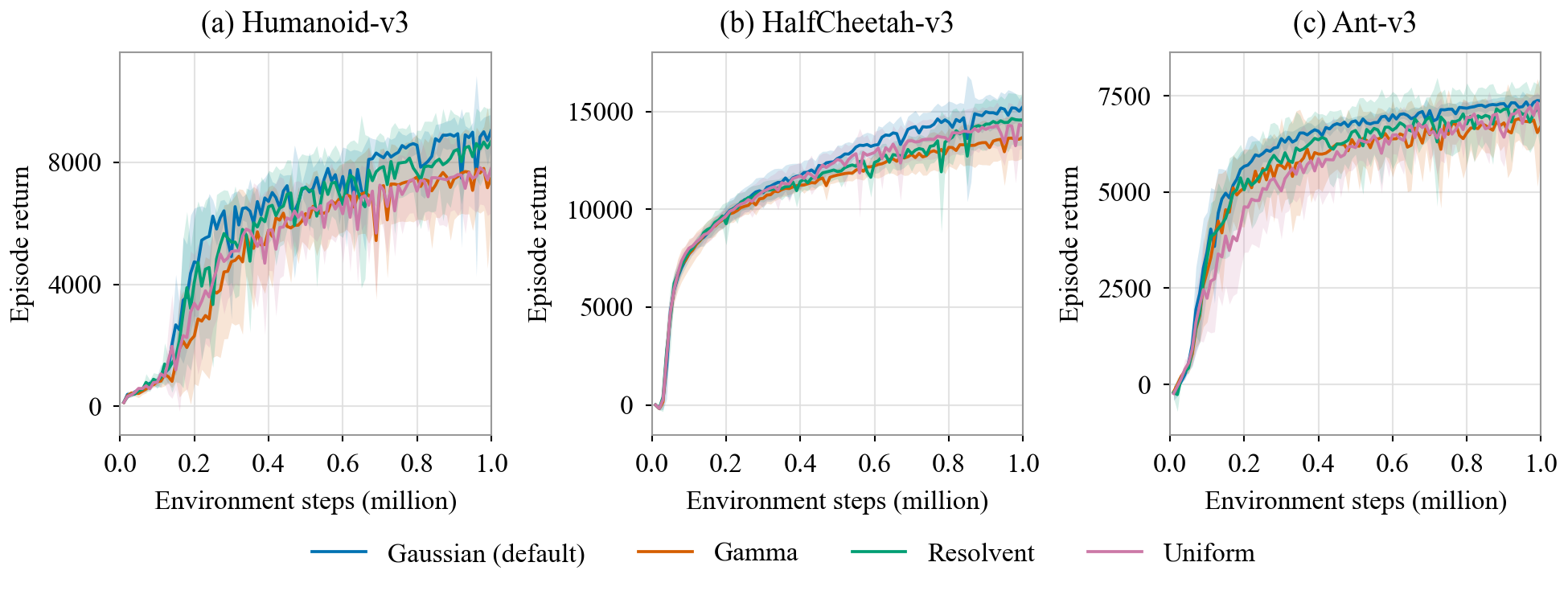}
  \caption{\textbf{Effect of spreading noise on MuJoCo.}
  The four noise choices use the same actor learning
  rates within each task.}
  \label{fig:experiment-cf-target-ablation}
\end{figure}

\paragraph{Sequential versus joint fitting.}
In addition to the Gaussian example in
Figure~\ref{fig:method-sequential-joint},
we also compare sequential and joint fitting
on MuJoCo.
On Humanoid-v3, sequential fitting achieves higher
 mean returns than joint fitting
across all tested loss weights.
The learning curves are
provided in Appendix~\ref{sec:experiment-joint-beta}.

\section{Conclusion and Limitations}
\label{sec:conclusion}

We introduced \ours{}, a Drifting-based approach to
learning expressive one-step policies for online RL.
Building on an energy-based view of Drifting, we derive
critic-guided Drift and entropy-induced Spread from
the MaxEnt objective.
The two updates are learned sequentially by the same
actor: Drift moves actions toward higher values,
while Spread matches a broader distribution constructed
from the updated actor.
This design incorporates the effect of entropy without
estimating policy densities or scores, while allowing
each action sample to be generated in one forward pass.
Under exact fitting, our analysis establishes convergence
to the fixed-critic Boltzmann target in the small-step,
long-time limit.
MuJoCo experiments show that \ours{} combines strong
control performance with efficient training and inference,
supporting Drifting as a practical approach to
generative policy learning.

Our work leaves two directions for future research.
First, our evaluation is limited to simulation
environments.
Testing on real-world control tasks would help
assess robustness to observation noise and changes
in system dynamics.
Second, we use a fixed spreading-noise distribution
and scale during training.
Adapting the spreading strength as learning progresses
may better balance exploration and value improvement.

\FloatBarrier %
\clearpage

\subsection*{AI Use Statement}

We used generative AI tools (GPT-5.6 Sol and GPT-6 Sol)
to draft and polish manuscript text, assist with coding and identify related work.
These tools were not used to generate synthetic data.
All AI-assisted material was reviewed by at least one
coauthor, and AI-assisted code was checked
manually.
The main technical contributions are the authors'
original work.
The authors take responsibility for the final content,
including the text, claims, code, and other research
artifacts produced with AI assistance.

\subsection*{Ethics Statement}

Our experiments use synthetic examples and MuJoCo simulations.
The study involves no human participants or personal or sensitive
data, and the method has not been deployed in safety-critical settings.
We are not aware of ethical concerns specific to this work beyond
those common to reinforcement learning research.

\subsection*{Reproducibility Statement}
\phantomsection
\label{sec:reproducibility}

{We provide the complete training algorithm for \ours{} in
Appendix~\ref{sec:method-critics}.}
Experimental settings and implementation details are provided in
Appendix~\ref{sec:appendix-experimental-details}.
We use publicly available MuJoCo simulation environments.
Source code and configurations for the main reinforcement-learning
experiments are available in an anonymous repository at
\url{https://anonymous.4open.science/r/DriftSpread-7A36}.
All experiments were run on an NVIDIA RTX 4090 GPU.

\bibliography{references}
\bibliographystyle{iclr2027_conference}
\clearpage
\appendix

\makeatletter
\newcommand{\AppendixLabelAlias}[2]{%
  \@ifundefined{r@#2}{}{%
    \protected@write\@auxout{}{\string\newlabel{#1}{\csname r@#2\endcsname}}%
  }%
  \@ifundefined{r@#2@cref}{}{%
    \protected@write\@auxout{}{\string\newlabel{#1@cref}{\csname r@#2@cref\endcsname}}%
  }%
}
\makeatother

\addtocontents{toc}{\protect\setcounter{tocdepth}{2}}

\begingroup
\edef\SavedTocDepth{\number\value{tocdepth}}
\renewcommand{\contentsname}{Appendix Contents}
\fontsize{9}{10.5}\selectfont
\setlength{\parskip}{0pt}
\hypersetup{hidelinks,linktoc=all}

\setcounter{tocdepth}{-1}

\makeatletter
\renewcommand*\l@section[2]{%
  \ifnum\c@tocdepth>\z@
    \addvspace{0.2em}%
    \begingroup
      \bfseries
      \@dottedtocline{1}{0pt}{1.5em}{#1}{#2}%
    \endgroup
  \fi
}
\renewcommand*\l@subsection{%
  \@dottedtocline{2}{1em}{2.4em}%
}
\makeatother

\tableofcontents

\setcounter{tocdepth}{\SavedTocDepth}
\endgroup

\section{Additional Method Details}
\label{sec:appendix-method}

We first derive Drift and Spread from the MaxEnt actor
objective. We then explain how the actor learns these
updates through action regression and distribution
matching, and give the complete training algorithm.

\subsection{Deriving Drift and Spread}
\label{sec:appendix-entropy-cf}
\label{sec:appendix-method-details}
\label{sec:method-supp-variation}

Fix a state $s$, critic $Q$, and temperature $\lambda>0$.
We consider actions $a\in\mathbb R^d$.
The negative MaxEnt policy-improvement objective is
\begin{equation}
    \mathcal F_{Q,s}(\pi)
    =-\E_{a\sim\pi(\cdot\mid s)}[Q(s,a)]
     -\lambda\mathcal H(\pi(\cdot\mid s)).
    \label{eq:appendix-action-energy}
\end{equation}
\AppendixLabelAlias{eq:method-action-energy}{eq:appendix-action-energy}
When $0<Z_Q(s)<\infty$, substituting the Boltzmann
density from \cref{eq:boltzmann-policy} gives
$\mathcal F_{Q,s}(\pi)
=\lambda D_{\mathrm{KL}}(\pi\|\pi_Q^*)
-\lambda\log Z_Q(s)$.
Thus, minimizing this energy over policy distributions
gives the Boltzmann target.

\paragraph{From the objective to distribution updates.}
Assume the policy density is smooth and positive,
the energy is finite, and the density decays sufficiently
fast for the following calculations.
Differentiating the energy with respect to the density
gives its first variation:
\begin{equation}
    \frac{\delta\mathcal F_{Q,s}}{\delta\pi}(\pi)(a)
    =-Q(s,a)+\lambda(1+\log\pi(a\mid s)).
    \label{eq:appendix-first-variation}
\end{equation}
\AppendixLabelAlias{eq:method-first-variation}{eq:appendix-first-variation}
The sample velocity given by the Wasserstein gradient
flow in \cref{eq:energy-velocity} is therefore
\begin{equation}
    V_\pi(a\mid s)
    =\nabla_aQ(s,a)-\lambda\nabla_a\log\pi(a\mid s).
    \label{eq:appendix-action-velocity}
\end{equation}
The critic gradient is directly available, but the
entropy term depends on $\nabla_a\log\pi$, which the
implicit actor does not provide.
To avoid computing this term, we consider how the
policy density changes.
Substituting the velocity into the continuity equation
and using $\pi\nabla_a\log\pi=\nabla_a\pi$ gives
\begin{equation}
    \partial_\tau\pi
    =-\nabla_a\cdot(\pi\nabla_aQ)
     +\lambda\Delta_a\pi.
    \label{eq:appendix-action-fokker-planck}
\end{equation}
The first term moves actions toward higher critic values.
The second spreads their distribution through diffusion.

\paragraph{Approximating Drift and Spread.}
For a step size $\eta>0$, Drift approximates the first
term by moving each action to
$a^{\mathrm D}=a+\eta\nabla_aQ(s,a)$.
Spread approximates the diffusion term by adding
independent noise
$\epsilon_\eta=\sqrt{2\lambda\eta}\,\zeta$
to the post-Drift actions.
The distribution of $\zeta$ does not depend on $\eta$
and satisfies
\begin{equation}
    \E[\zeta]=0,
    \qquad
    \E[\zeta\zeta^\top]=I_d,
    \qquad
    \E\|\zeta\|^3<\infty.
    \label{eq:appendix-noise-conditions}
\end{equation}
The added noise therefore has mean zero and covariance
$2\lambda\eta I_d$; it need not be Gaussian.

To see why this approximates diffusion, consider a smooth
function $g$ with bounded derivatives through order three.
A second-order Taylor expansion gives
\begin{equation}
    \E_\zeta[g(a+\epsilon_\eta)]
    =g(a)+\lambda\eta\Delta_a g(a)+O(\eta^{3/2}).
    \label{eq:appendix-spread-expansion}
\end{equation}
The zero mean removes the linear term, and the covariance
determines the quadratic term.
To first order in $\eta$, this is the change produced
by diffusion with coefficient $\lambda$.

Applying Drift followed by Spread gives
\begin{equation}
    a^+=a+\eta\nabla_aQ(s,a)+\epsilon_\eta,
    \qquad
    \epsilon_\eta=\sqrt{2\lambda\eta}\,\zeta.
    \label{eq:method-langevin-step}
\end{equation}
Noise draws are independent across cycles and independent
of the current action.
Appendix~\ref{sec:appendix-cf-targets} gives the noise
distributions used in our experiments.

Gaussian noise is a useful special case.
Adding $\epsilon_\eta\sim\mathcal N(0,2\lambda\eta I_d)$
gives the exact entropy-only diffusion over duration
$\eta$.
The other listed choices approximate this diffusion
and share its small-step limit.

{Appendix~\ref{sec:appendix-theory} gives the convergence analysis.}
We next describe how the actor learns the Drift targets
through regression and the Spread target through
distribution matching.

\subsection{Computing the Actor Losses}
\label{sec:appendix-actor-implementation}
\label{sec:appendix-spread-fitting}

Drift provides an updated action for each generator
input, so it can be learned through regression.
Spread specifies a broader action distribution.
We explain why this requires a different loss, then
give the sample-based losses used in training.
Throughout this subsection, the actor generates
$a=f_\theta(s,z)$ with $z\sim p_{\mathrm{ref}}$.

\subsubsection{Why Spread Uses Distribution Matching}
\label{sec:method-supp-pairing}
\label{sec:method-noise-history}
\label{sec:appendix-background-wgf}

Spread trains the actor to represent the distribution
produced by adding noise to its post-Drift actions.
This allows the next Drift update to start from the
broader distribution.

Fix a state $s$ and write
$\bar a(z)=\sg[f_{\theta_{k+1/2}}(s,z)]$.
If we apply regression to the noisy target
$\bar a(z)+\epsilon_\eta$ using the same input $z$,
then
\begin{equation}
    \begin{aligned}
        \E_{z,\epsilon_\eta}
        \|f_\theta(s,z)-\bar a(z)-\epsilon_\eta\|^2
        &=
        \E_z\|f_\theta(s,z)-\bar a(z)\|^2
        +\E\|\epsilon_\eta\|^2 \\
        &=
        \E_z\|f_\theta(s,z)-\bar a(z)\|^2
        +2\lambda\eta d.
    \end{aligned}
    \label{eq:method-cf-pairing}
\end{equation}
\AppendixLabelAlias{eq:wppg-noisy-regression}{eq:method-cf-pairing}
The cross term vanishes because the noise is independent
of $z$ and has zero mean.
Its remaining contribution is a constant, so the
post-Drift actor already minimizes this regression loss.
The loss therefore does not teach the actor to reproduce
the added spread.

We instead use the Spread MMD loss in
\cref{eq:method-mmd-population} to match the actor and
target distributions.

\subsubsection{Losses Computed from Actor Samples}
\label{sec:appendix-cf-fitting}
\label{sec:method-supp-fitting}

\paragraph{Drift regression.}
Sample $B$ states $s_b$ from the replay buffer and
$N_{\mathrm D}$ independent inputs
$z^{\mathrm D}_{b,n}\sim p_{\mathrm{ref}}$ per state.
For each current action
$a^k_{b,n}=f_{\theta_k}(s_b,z^{\mathrm D}_{b,n})$,
we construct the fixed target
\begin{equation}
    \widetilde a^{k+1/2}_{b,n}
    =\sg\!\left[
        a^k_{b,n}
        +\eta\nabla_aQ(s_b,a^k_{b,n})
    \right].
    \label{eq:appendix-drift-target}
\end{equation}
The empirical Drift loss is
\begin{equation}
    \widehat{\mathcal L}_{\mathrm{Drift}}
    (\theta;\theta_k)
    =
    \frac{1}{2BN_{\mathrm D}}
    \sum_{b=1}^{B}\sum_{n=1}^{N_{\mathrm D}}
    \left\|
        f_\theta(s_b,z^{\mathrm D}_{b,n})
        -\widetilde a^{k+1/2}_{b,n}
    \right\|^2.
    \label{eq:method-drift-empirical}
\end{equation}
The actor prediction and its target use the same input.
The stop-gradient operation holds the target fixed,
so only the prediction contributes gradients to the
actor parameters.
One optimizer step from $\theta_k$ gives
$\theta_{k+1/2}$.

\paragraph{Computing Spread with characteristic functions.}
We construct the Spread target from this updated actor.
Write $K_\eta=K_{\lambda,\eta}$ for the spreading-noise
distribution and $m_\eta=\widehat K_{\lambda,\eta}$
for its characteristic function (CF).
By independence, the target CF is
$m_\eta(\omega)
\varphi_{\theta_{k+1/2}}(\omega\mid s)$.
We use this expression to compute the Spread MMD loss
as in \cref{eq:method-cf-population}.

For each replay state, draw $N_{\mathrm{CF}}$ fresh
inputs $z^{\mathrm{CF}}_{b,n}\sim p_{\mathrm{ref}}$.
Draw $M_\omega$ frequencies $\omega_\ell\sim\rho$,
shared across the batch.
For any actor parameters $\theta$, estimate its CF by
\begin{equation}
    \widehat\varphi_\theta(\omega_\ell\mid s_b)
    =
    \frac{1}{N_{\mathrm{CF}}}
    \sum_{n=1}^{N_{\mathrm{CF}}}
    e^{i\omega_\ell^\top
       f_\theta(s_b,z^{\mathrm{CF}}_{b,n})}
    =
    \widehat c^\theta_{b,\ell}
    +i\widehat s^\theta_{b,\ell}.
    \label{eq:method-cf-moments}
\end{equation}
Here, $\widehat c^\theta$ and $\widehat s^\theta$
are the averages of the corresponding cosines and sines.
Using the post-Drift actor, form the fixed target
\begin{equation}
    \widehat\varphi^{\mathrm{tar}}_{b,\ell}
    =
    \sg\!\left[
        m_\eta(\omega_\ell)
        \widehat\varphi_{\theta_{k+1/2}}
        (\omega_\ell\mid s_b)
    \right]
    =
    \widehat c^{\mathrm{tar}}_{b,\ell}
    +i\widehat s^{\mathrm{tar}}_{b,\ell}.
    \label{eq:method-target-moments}
\end{equation}
\AppendixLabelAlias{eq:method-empirical-cf}{eq:method-cf-moments}
The actor and target use the same inputs and frequencies.
Their squared CF difference can be computed with
real-valued operations:
\begin{equation}
    \widehat{\mathcal L}_{\mathrm{Spread}}
    (\theta;\theta_{k+1/2})
    =
    \frac{1}{BM_\omega}
    \sum_{b=1}^{B}\sum_{\ell=1}^{M_\omega}
    \left[
        (\widehat c^\theta_{b,\ell}
         -\widehat c^{\mathrm{tar}}_{b,\ell})^2
        +
        (\widehat s^\theta_{b,\ell}
         -\widehat s^{\mathrm{tar}}_{b,\ell})^2
    \right].
    \label{eq:method-cf-empirical}
\end{equation}
One optimizer step from $\theta_{k+1/2}$ gives
$\theta_{k+1}$.
We compare the CFs separately at each state before
averaging over the batch.
{This is a finite-sample approximation to the population
MMD objective.}

\paragraph{Using the noise characteristic function.}
The known noise CF lets us average over the spreading
noise analytically.
For fixed post-Drift actions
$a_1,\ldots,a_{N_{\mathrm{CF}}}$ and independent
$\epsilon_n\sim K_\eta$,
\begin{equation}
    \E_{\epsilon_1,\ldots,\epsilon_{N_{\mathrm{CF}}}}
    \left[
        \frac{1}{N_{\mathrm{CF}}}
        \sum_{n=1}^{N_{\mathrm{CF}}}
        e^{i\omega^\top(a_n+\epsilon_n)}
    \right]
    =
    m_\eta(\omega)
    \frac{1}{N_{\mathrm{CF}}}
    \sum_{n=1}^{N_{\mathrm{CF}}}
    e^{i\omega^\top a_n}.
    \label{eq:appendix-cf-noise-average}
\end{equation}
We therefore construct the target CF by multiplying
the estimated actor CF by $m_\eta$.
Computing the inner products and sample averages costs
$O(BN_{\mathrm{CF}}M_\omega d)$ operations, in addition
to the actor's forward and backward passes.

\subsubsection{Spreading Noise and Characteristic Functions}
\label{sec:appendix-cf-targets}

We use Gaussian spreading noise by default and compare
three alternatives in
{Figure~\ref{fig:experiment-cf-target-ablation}}.
All four choices use the same actor and loss.
For a post-Drift distribution
$\pi_{k+1/2}(\cdot\mid s)$, the Spread target is
\begin{equation}
    \widetilde\pi_{k+1}
    =K_\eta*\pi_{k+1/2},
    \qquad
    \varphi_{\widetilde\pi_{k+1}}(\omega\mid s)
    =
    m_\eta(\omega)
    \varphi_{\pi_{k+1/2}}(\omega\mid s).
    \label{eq:cf-general-target}
\end{equation}
The noise distribution $K_\eta$ determines the target.
The frequency distribution $\rho$ determines the MMD
kernel used to compare that target with the actor.

\paragraph{Noise distributions and their characteristic functions.}
Write $c=\lambda\eta$.
For the Gaussian and mixture constructions, let
$\xi\sim\mathcal N(0,I_d)$ and
$v_\kappa\sim\operatorname{Gamma}
(\kappa,\text{rate }\kappa)$ be independent,
with fixed $\kappa>0$.
For uniform noise, let the coordinates of $U$ be
independent and uniform on $[-1,1]$.
The resulting noise distributions are listed in
Table~\ref{tab:appendix-cf-target-families}.
Each has mean zero and covariance $2cI_d$.

\begin{table}[!htbp]
    \centering
    \small
    \setlength{\tabcolsep}{6pt}
    \caption{Spreading noise and characteristic functions,
    with $c=\lambda\eta$ and $a_\eta=\sqrt{6c}$.
    The Gamma construction scales a Gaussian vector
    by an independent Gamma-distributed variance factor.}
    \label{tab:appendix-cf-target-families}
    \begin{tabular}{@{}lcc@{}}
        \toprule
        Noise
        & Characteristic function $m_\eta(\omega)$
        & Noise sample $\epsilon_\eta$ \\
        \midrule
        Gaussian
        & $e^{-c\|\omega\|^2}$
        & $\sqrt{2c}\,\xi$ \\
        Resolvent
        & $(1+c\|\omega\|^2)^{-1}$
        & $\sqrt{2cv_1}\,\xi$ \\
        Gamma mixture
        & $(1+c\|\omega\|^2/\kappa)^{-\kappa}$
        & $\sqrt{2cv_\kappa}\,\xi$ \\
        Uniform
        & $\prod_{j=1}^{d}
           \operatorname{sinc}(a_\eta\omega_j)$
        & $a_\eta U$ \\
        \bottomrule
    \end{tabular}
\end{table}

For the Gamma mixture, conditioning on $v_\kappa$ gives
\begin{equation}
    \E e^{i\omega^\top\sqrt{2cv_\kappa}\,\xi}
    =
    \E e^{-cv_\kappa\|\omega\|^2}
    =
    \left(
        1+\frac{c\|\omega\|^2}{\kappa}
    \right)^{-\kappa}.
    \label{eq:cf-gamma-family}
\end{equation}
\AppendixLabelAlias{eq:cf-resolvent-mixture}{eq:cf-gamma-family}
Resolvent is the case $\kappa=1$, for which
$v_1\sim\operatorname{Exp}(1)$.
The Uniform CF follows by multiplying the CFs of
its independent coordinates, where
$\operatorname{sinc}(x)=\sin(x)/x$ and
$\operatorname{sinc}(0)=1$.

\paragraph{Connection to the diffusion update.}
Gaussian noise gives the exact entropy-only diffusion
over duration $\eta$.
The other choices agree with it to first order:
for each fixed $\omega$,
\begin{equation}
    m_\eta(\omega)
    =1-\lambda\eta\|\omega\|^2+o(\eta).
    \label{eq:cf-first-order}
\end{equation}
All four noise distributions can be written as
$\epsilon_\eta=\sqrt{2\lambda\eta}\,\zeta$,
where $\zeta$ satisfies
\cref{eq:appendix-noise-conditions}.
They therefore satisfy the noise assumptions in
\cref{thm:cf-target-convergence} and share the same
diffusion limit.

\subsubsection{Choosing the Frequency Scale}
\label{sec:appendix-frequency}
\label{sec:appendix-frequency-gaussian}

We sample Gaussian frequencies
$\omega=\sigma_\omega Z$,
where $Z\sim\mathcal N(0,I_d)$,
for every spreading-noise choice.
The corresponding MMD kernel is
$h(a,a')
=\exp(-\sigma_\omega^2\|a-a'\|^2/2)$.
Thus, the frequency scale determines the Gaussian
kernel's bandwidth.

For Gaussian spreading noise,
$m_\eta(\omega)=e^{-\lambda\eta\|\omega\|^2}$.
If $\lambda\eta\|\omega\|^2$ is very small,
the target CF is almost unchanged by Spread.
We choose a default scale that accounts for both
the spreading-noise size and the action dimension.
For independent
$\epsilon_{\mathrm G}\sim
\mathcal N(0,2\lambda\eta I_d)$ and $\omega$,
we set
\begin{equation}
    \E[(\omega^\top\epsilon_{\mathrm G})^2]
    =2\lambda\eta d\,\sigma_{\mathrm G}^2=1,
    \qquad
    \sigma_{\mathrm G}
    =(2\lambda\eta d)^{-1/2}.
    \label{eq:frequency-gaussian-scale}
\end{equation}

\subsection{Complete Training Algorithm}
\label{sec:appendix-training}
\label{sec:method-critics}

Algorithm~\ref{alg:rl-loop} combines the actor updates
with off-policy critic learning.
Environment transitions are stored in a replay buffer
$\mathcal D$ and reused for training.
Each gradient iteration updates the critics first,
then performs one Drift step and one Spread step
using the replay states.

\paragraph{Critic update.}
Following \cref{sec:method-practical}, the critics
estimate discounted rewards without an entropy term
in the Bellman target.
Entropy regularization remains in the actor objective.

Let $Q_{\phi_1},Q_{\phi_2}$ be the current critics,
$\bar Q_1,\bar Q_2$ their target networks, and
$f_{\bar\theta}$ the target actor.
For each replay transition
$(s_b,a_b,r_b,s'_b,\delta_b)$,
draw $z'_b\sim p_{\mathrm{ref}}$ and set
$a'_b=f_{\bar\theta}(s'_b,z'_b)$.
Both critics regress toward the fixed target
\begin{equation}
    y_b
    =
    \sg\!\left[
        r_b+\gamma(1-\delta_b)
        \min_{i\in\{1,2\}}
        \bar Q_i(s'_b,a'_b)
    \right].
    \label{eq:appendix-critic-target}
\end{equation}
Using the smaller target value reduces overestimation
\citep{fujimoto2018addressing}.
Each critic minimizes
$B^{-1}\sum_b
(Q_{\phi_i}(s_b,a_b)-y_b)^2$.
Here, $\delta_b=1$ for an environment termination;
time-limit truncations retain bootstrapping when
required by the task.
This update requires no policy-density estimate.

\paragraph{Actor update.}
After updating the critics, we hold their parameters
fixed and use
$Q(s,a)=\min_{i\in\{1,2\}}Q_{\phi_i}(s,a)$
to construct the Drift targets.
One optimizer step on
$\widehat{\mathcal L}_{\mathrm{Drift}}$
gives $\theta_{k+1/2}$.
We then construct the Spread target from this updated
actor and take one step on
$\widehat{\mathcal L}_{\mathrm{Spread}}$
to obtain $\theta_{k+1}$.
The evolution step $\eta$ controls the target
displacement and spreading-noise scale; it is
separate from the optimizer learning rates.

\paragraph{Target-network update.}
{On the target-update schedule, we move each target
parameter toward the corresponding online parameter:}
\begin{equation}
    \bar\theta
    \leftarrow
    (1-\tau_{\mathrm{targ}})\bar\theta
    +\tau_{\mathrm{targ}}\theta_{k+1},
    \qquad
    \bar\phi_i
    \leftarrow
    (1-\tau_{\mathrm{targ}})\bar\phi_i
    +\tau_{\mathrm{targ}}\phi_i.
    \label{eq:appendix-target-update}
\end{equation}
Here, $\bar\phi_i$ denotes the parameters of $\bar Q_i$.

\begin{algorithm}[H]
    \caption{\ours{}: online training}
    \label{alg:rl-loop}
    \begin{algorithmic}[1]
        \Require Actor $f_\theta$,
        critics $Q_{\phi_1},Q_{\phi_2}$,
        and replay buffer $\mathcal D$
        \State Initialize target networks from the actor
        and critics; set $k\gets0$ and observe state $s$.
        \While{training continues}
            \State Sample $z\sim p_{\mathrm{ref}}$
            and take action $a\gets f_{\theta_k}(s,z)$.
            \State Observe $(r,s',\delta)$ and store
            $(s,a,r,s',\delta)$ in $\mathcal D$.
            \For{each scheduled gradient iteration}
                \State Sample a batch of $B$ transitions
                from $\mathcal D$.
                \State Compute the targets in
                \cref{eq:appendix-critic-target}
                and update both critics by
                squared-error regression.
                \State Set
                $Q(s,a)\gets
                \min_{i\in\{1,2\}}Q_{\phi_i}(s,a)$
                and hold the critics fixed.
                \State Draw $N_{\mathrm D}$ actor inputs
                per replay state and construct the
                targets in \cref{eq:appendix-drift-target}.
                \State Take one step on
                $\widehat{\mathcal L}_{\mathrm{Drift}}$
                to obtain $\theta_{k+1/2}$.
                \State Draw $N_{\mathrm{CF}}$ fresh actor
                inputs per state and $M_\omega$ frequencies.
                \State Construct the fixed Spread target
                from $f_{\theta_{k+1/2}}$ using
                \cref{eq:method-target-moments}.
                \State Take one step on
                $\widehat{\mathcal L}_{\mathrm{Spread}}$
                to obtain $\theta_{k+1}$.
                \State Update the target networks
                on schedule; set $k\gets k+1$.
            \EndFor
            \If{the episode has ended}
                \State Reset the environment
                and observe $s$.
            \Else
                \State $s\gets s'$.
            \EndIf
        \EndWhile
    \end{algorithmic}
\end{algorithm}

\section{Boltzmann Convergence}
\label{sec:appendix-theory}
\label{sec:method-supp-guarantees}

We prove the convergence result in
\cref{thm:main-fixed-critic-convergence}.
The argument connects the actor losses to the Boltzmann target
in three steps. Exact fitting makes each Drift--Spread cycle
equivalent to moving an action along the critic gradient and
then adding noise. As the step size decreases, these updates
approach a continuous diffusion. Over time, this diffusion
converges to the Boltzmann target.

\subsection{Setting and Assumptions}
\label{sec:theory-setting}
\label{sec:appendix-background-maxent}

Fix a state $s$, critic $Q$, and temperature $\lambda>0$.
We work on $\mathbb R^d$ and omit state conditioning from
the policy notation. Write $b(a)=\nabla_a Q(s,a)$.
The target is the Boltzmann distribution
$\pi_Q^*(a)=Z_Q(s)^{-1}\exp(Q(s,a)/\lambda)$, where
$Z_Q(s)=\int_{\mathbb R^d}\exp(Q(s,a)/\lambda)\,da$.

\begin{assumption}[Conditions on the critic]
\label{ass:theory-flow}
The target is normalizable, $0<Z_Q(s)<\infty$.
The critic $Q(s,\cdot)$ is twice continuously differentiable,
and there is a finite constant $C\ge0$ such that
\begin{equation}
    \langle a,b(a)\rangle\le C(1+\|a\|^2)
    \qquad\text{for all }a\in\mathbb R^d.
    \label{eq:theory-radial}
\end{equation}
\end{assumption}
The smoothness condition makes $b$ locally Lipschitz.
The bound in \cref{eq:theory-radial} controls how strongly
the drift pushes actions away from the origin and prevents
the continuous process from escaping to infinity in finite time.

For step size $\eta>0$, the Spread noise is
$\epsilon_{\eta,k}=\sqrt{2\lambda\eta}\,\zeta_k$.
The variables $\zeta_k$ are independent copies of a random
vector $\zeta$ whose distribution does not depend on $\eta$,
with
\begin{equation}
    \E\zeta=0,\qquad
    \E[\zeta\zeta^\top]=I_d,\qquad
    \E\|\zeta\|^3<\infty.
    \label{eq:theory-noise-moments}
\end{equation}
Each noise draw is independent of the current action.
These conditions include Gaussian noise and the other
noise choices in Appendix~\ref{sec:appendix-cf-targets}.
They allow different finite-step updates while giving
the same small-step limit.

\subsection{Exact Fitting Recovers the Action Update}
\label{sec:theory-fitting}

We first show what minimizing the two actor losses achieves.
Let $\pi_k^{(\eta)}$, $\pi_{k+1/2}^{(\eta)}$, and
$\pi_{k+1}^{(\eta)}$ denote the actor distributions before
Drift, after Drift, and after Spread, respectively.
Each target is held fixed during its update, and the
Spread target is constructed from the post-Drift actor.

\begin{proposition}[Action update under exact fitting]
\label{prop:theory-exact-fitting}
Suppose that, at the fixed state $s$, the population
Drift regression loss and Spread MMD loss are zero
after their respective updates in every cycle.
{If the MMD kernel $h$ is characteristic, meaning that zero MMD implies identical distributions,}
then $\pi_k^{(\eta)}$ is the distribution
of the random action sequence
\begin{equation}
    A_{k+1}^{(\eta)}
    =A_k^{(\eta)}+\eta b(A_k^{(\eta)})
     +\sqrt{2\lambda\eta}\,\zeta_k,
    \qquad A_0^{(\eta)}\sim\pi_0.
    \label{eq:theory-action-update}
\end{equation}
\end{proposition}

\begin{proof}
Zero Drift regression loss means that, for almost every
generator input $z$, the updated actor reproduces its target:
\[
    f_{\theta_{k+1/2}}(s,z)
    =f_{\theta_k}(s,z)
     +\eta b\bigl(f_{\theta_k}(s,z)\bigr).
\]
Thus, if $A_k^{(\eta)}$ follows the current policy,
the post-Drift policy is the distribution of
$A_k^{(\eta)}+\eta b(A_k^{(\eta)})$.

Spread then targets the distribution obtained by adding
independent noise to these post-Drift actions.
{Denote this distribution by $\widetilde\pi_{k+1}$.
By the population MMD expression in
\cref{eq:method-mmd-population}, zero Spread loss gives
\[
    \operatorname{MMD}_h^2
    \bigl(\pi_{k+1}^{(\eta)},\widetilde\pi_{k+1}\bigr)=0.
\]
Since $h$ is characteristic, its kernel mean embedding
is injective \citep{gretton2012kernel}, so
$\pi_{k+1}^{(\eta)}=\widetilde\pi_{k+1}$.}
Combining the two updates proves
\cref{eq:theory-action-update}, and repeating this argument
establishes the result for every cycle.
\end{proof}

\subsection{Convergence to the Boltzmann Target}
\label{sec:theory-fixed-critic}
\label{sec:theory-cf-target-convergence}

Each cycle in \cref{eq:theory-action-update} advances
the evolution time by $\eta$.
To compare different step sizes at the same time $\tau$,
we take $k=\lfloor\tau/\eta\rfloor$ and use the same
initial policy $\pi_0$ for every $\eta$.
The limiting continuous process is
\begin{equation}
    dA_\tau=b(A_\tau)\,d\tau
        +\sqrt{2\lambda}\,dW_\tau,
    \qquad A_0\sim\pi_0,
    \label{eq:theory-sde}
\end{equation}
where $W_\tau$ is standard $d$-dimensional Brownian motion,
independent of $A_0$.
This is the Langevin diffusion whose density follows
the MaxEnt evolution in \cref{eq:method-fokker-planck}.
Let $p_\tau$ denote its distribution at time $\tau$.

\begin{theorem}[Boltzmann convergence under exact fitting]
\label{thm:cf-target-convergence}
Under \cref{ass:theory-flow}, the noise conditions above,
and the exact-fitting conditions in
\cref{prop:theory-exact-fitting}, for every initial
policy $\pi_0$ and every fixed finite time $\tau$,
\begin{equation}
    \pi_{\lfloor\tau/\eta\rfloor}^{(\eta)}
    \Rightarrow p_\tau
    \qquad\text{as }\eta\downarrow0.
    \label{eq:cf-target-finite-time}
\end{equation}
Moreover, $p_\tau\Rightarrow\pi_Q^*$ as $\tau\to\infty$.
Consequently,
\begin{equation}
    \lim_{\tau\to\infty}\lim_{\eta\downarrow0}
    \pi_{\lfloor\tau/\eta\rfloor}^{(\eta)}
    =\pi_Q^*
    \qquad\text{weakly}.
    \label{eq:cf-target-boltzmann-limit}
\end{equation}
Here, $\Rightarrow$ denotes weak convergence:
expectations of every bounded continuous function converge.
\end{theorem}

\begin{proof}
\emph{Small-step limit.}
First fix an initial action $A_0=x$ and a finite time
interval $[0,H]$.
The drift is locally Lipschitz, so the diffusion has
a unique solution until a possible escape to infinity.
Applying It\^o's formula to $1+\|A_\tau\|^2$,
with the process stopped when it leaves a ball,
and using \cref{eq:theory-radial} rules out such
an escape in finite time.
Thus the solution exists throughout $[0,H]$.

The accumulated noise in the discrete updates is
\[
    W_{k\eta}^{(\eta)}
    =\sqrt{\eta}\sum_{j=0}^{k-1}\zeta_j.
\]
Interpolate these values linearly between grid times.
The functional central limit theorem implies that
$W^{(\eta)}$ converges in distribution to Brownian motion
on $[0,H]$.
Along any sequence of step sizes tending to zero,
the Skorokhod representation theorem allows these paths
to be constructed on a common probability space so that
they converge uniformly almost surely
\citep{billingsley1999convergence}.

We next show that this convergence of the noise paths
also gives convergence of the action paths.
For each positive integer $R$, smoothly modify $b$
outside the ball of radius $R$ so that the modified
drift $b_R$ is bounded and globally Lipschitz.
Use the same noise paths to form both the discrete
updates and the diffusion with this modified drift.
Within each time step, the change in the interpolated
action is bounded by the drift increment and the noise
increment, both of which tend uniformly to zero.
Subtracting the two path equations and applying
Gr\"onwall's inequality therefore shows that the
interpolated discrete action paths converge uniformly
to the diffusion paths with drift $b_R$.

The original diffusion stays within a finite ball on
$[0,H]$ almost surely.
Choose an integer $R$ larger than its maximum norm
by at least one.
For this $R$, modifying the drift has no effect on
the diffusion path.
Uniform convergence implies that the corresponding
discrete paths also remain inside the ball of radius
$R$ for all sufficiently small step sizes.
They therefore follow the original update as well.
This proves convergence of the original action paths.
By continuity of $A_\tau$, it also gives
$A_{\lfloor\tau/\eta\rfloor}^{(\eta)}\Rightarrow A_\tau$
for every fixed $\tau\le H$.

The argument holds for every initial action $x$.
Integrating bounded continuous test functions over
$x\sim\pi_0$ and applying bounded convergence extends
the result to any initial policy.
\Cref{prop:theory-exact-fitting} then gives
\cref{eq:cf-target-finite-time} for the fitted actors.

\emph{Long-time limit.}
The Boltzmann density satisfies
\[
    \lambda\nabla_a\pi_Q^*(a)
    =\pi_Q^*(a)b(a).
\]
Thus the drift and diffusion terms balance at
$\pi_Q^*$, making it a stationary distribution
of \cref{eq:theory-sde}.
To establish convergence to this distribution,
we apply \citet[Theorem~2.1]{roberts1996exponential}.
Its conditions hold because $\pi_Q^*$ is positive,
$\nabla_a\log\pi_Q^*=b/\lambda$ is continuously
differentiable, and
\[
    \langle a,\nabla_a\log\pi_Q^*(a)\rangle
    \le\frac{C}{\lambda}(1+\|a\|^2).
\]
With time rescaled by $2\lambda$, their diffusion
is exactly \cref{eq:theory-sde}.
The theorem gives convergence to $\pi_Q^*$ from
every initial action.
Averaging over $\pi_0$ gives $p_\tau\Rightarrow\pi_Q^*$.
Combining this result with
\cref{eq:cf-target-finite-time} proves
\cref{eq:cf-target-boltzmann-limit} and hence
\cref{thm:main-fixed-critic-convergence}.
\end{proof}

\FloatBarrier
\section{Sequential versus Joint Fitting}
\label{sec:appendix-design-analysis}

Sequential fitting first learns the Drift target and then
constructs the Spread target from the updated actor.
A joint alternative constructs both targets from the actor
before Drift and minimizes
$\mathcal L_{\mathrm{Drift}}
+\beta\mathcal L_{\mathrm{Spread}}^{\mathrm{old}}$.
Here, $\beta>0$ weights the Spread MMD loss, and
``old'' indicates that its target uses the pre-Drift actor.

{An action can have multiple dimensions.
A value of $\beta$ that balances Drift and Spread
in one dimension may place too much or too little
weight on Spread in other dimensions.
Because $\beta$ scales the entire Spread loss, adjusting it
cannot always balance all dimensions simultaneously.
Using a separate weight for each dimension would also
make hyperparameter tuning much harder, as action spaces
can have dozens or even hundreds of dimensions.
The limitation of a single $\beta$ can arise even
for simple action distributions.
We construct a two-dimensional Gaussian example in which
no choice of the scalar weight $\beta$ can balance Drift and
Spread in both dimensions at the Boltzmann target
(see \cref{prop:theory-joint-counterexample}).
Appendix~\ref{sec:experiment-joint-beta} compares Sequential
and Joint in online policy learning.}

\subsection{A Two-dimensional Gaussian Example}
\label{sec:theory-joint-sequential}
\label{sec:theory-joint-counterexample}

{We use a target that the actor can represent exactly
to isolate the effect of scalar joint weighting.}
Fix one state, omit state conditioning, and set
$\lambda=1$ and $\kappa=(16,1/4)^\top$.
The critic and its Boltzmann target are
\begin{equation}
    Q(a)=-\frac12\sum_{i=1}^2\kappa_i a_i^2,
    \qquad
    \pi_Q^*=\mathcal N\!\left(
        0,\operatorname{diag}(1/16,4)
    \right).
    \label{eq:toy-critic}
\end{equation}
The target is narrow in the first action coordinate
and wide in the second.
Our actor is $f_\sigma(z)=\operatorname{diag}(\sigma_1,\sigma_2)z$,
where $z\sim\mathcal N(0,I_2)$ and $\sigma_i>0$.
Its distribution is denoted by $\pi_\sigma$, and the
target is represented by $\sigma^*=(1/4,2)^\top$.

\paragraph{The two fitting procedures.}
Take $0<\eta<1/16$ and Gaussian Spread noise
$\mathcal N(0,2\eta I_2)$.
For an actor with standard deviations $\sigma$,
Drift multiplies coordinate $i$ by $1-\eta\kappa_i$,
while adding the Spread noise increases its variance
by $2\eta$.
If both targets are constructed from this actor,
their standard deviations are
\begin{equation}
    \sigma_{\mathrm D,i}=(1-\eta\kappa_i)\sigma_i,
    \qquad
    \sigma_{\mathrm S,i}^{\mathrm{old}}
    =\sqrt{\sigma_i^2+2\eta}.
    \label{eq:toy-targets}
\end{equation}

Sequential first fits $\sigma_{\mathrm D}$ and then
fits the distribution obtained by adding noise to
that updated actor.
Joint instead fits both targets in \cref{eq:toy-targets}
at once.
For candidate standard deviations $\sigma'$, its loss is
\begin{equation}
    \mathcal J_{\eta,\beta}(\sigma';\sigma)
    =\frac12\|\sigma'-\sigma_{\mathrm D}\|^2
     +\beta\mathcal M\bigl(
        \sigma',\sigma_{\mathrm S}^{\mathrm{old}}
     \bigr),
    \label{eq:toy-joint-loss}
\end{equation}
where $\mathcal M(v,w)=\operatorname{MMD}_h^2(\pi_v,\pi_w)$
uses the Gaussian kernel
$h(a,a')=\exp(-\|a-a'\|^2/2)$.
The first term is the Drift regression loss for this actor.
The reference $\sigma$, and hence both targets,
are held fixed when optimizing $\sigma'$.

If repeated exact joint fitting leaves an actor unchanged,
that actor must satisfy
\begin{equation}
    \left.
    \nabla_{\sigma'}\mathcal J_{\eta,\beta}(\sigma';\sigma)
    \right|_{\sigma'=\sigma}=0.
    \label{eq:toy-joint-fixed-point-condition}
\end{equation}
We call a positive solution of this equation a stationary
point of Joint.
The following result shows that such points remain
different from the target, even when the weight is
retuned as the step size decreases.

\begin{proposition}[A persistent mismatch under joint fitting]
\label{prop:theory-joint-counterexample}
For any positive weights $\beta_\eta$ and corresponding
positive stationary points $\sigma_\eta$ of Joint,
\begin{equation}
    \liminf_{\eta\downarrow0}
    D_{\mathrm{KL}}(\pi_{\sigma_\eta}\|\pi_Q^*)
    \ge c_{\mathrm{stat}}\approx0.50735.
    \label{eq:theory-joint-stationary-gap}
\end{equation}
In contrast, repeated exact Sequential fitting converges
to a Gaussian distribution with variances
\begin{equation}
    \sigma_{i,\mathrm{seq}}^2
    =\frac{1}{\kappa_i(1-\kappa_i\eta/2)}
    \longrightarrow\frac{1}{\kappa_i}
    \qquad\text{as }\eta\downarrow0.
    \label{eq:theory-joint-example-stationary}
\end{equation}
Its stationary KL divergence therefore tends to zero.
\end{proposition}

\begin{proof}
\emph{Sequential fitting.}
Let $x_{i,k}$ denote the variance of coordinate $i$
before cycle $k$.
Exact Drift fitting contracts this variance, and
the subsequent Spread fitting adds $2\eta$:
\[
    x_{i,k+1}=(1-\eta\kappa_i)^2x_{i,k}+2\eta.
\]
Because $0<1-\eta\kappa_i<1$, this recursion converges
to the fixed variance in
\cref{eq:theory-joint-example-stationary}.
These variances approach the Boltzmann target variances
$1/\kappa_i$ as $\eta\downarrow0$.

\emph{Joint fitting.}
For the Gaussian actor and kernel used here, the
MMD loss has the closed form
\begin{equation}
\begin{aligned}
    \mathcal M(v,w)
    &=\prod_{i=1}^2(1+2v_i^2)^{-1/2}
      +\prod_{i=1}^2(1+2w_i^2)^{-1/2} \\
    &\quad
      -2\prod_{i=1}^2(1+v_i^2+w_i^2)^{-1/2}.
\end{aligned}
    \label{eq:theory-joint-example-cf}
\end{equation}
To write its gradient compactly, set
$x_i=\sigma_i^2$, $r_i=1+2x_i$,
$G_0=(r_1r_2)^{-1/2}$, and
$G_\eta=((r_1+2\eta)(r_2+2\eta))^{-1/2}$.
Differentiate \cref{eq:toy-joint-loss} with respect to
the candidate $\sigma'$, keeping the reference $\sigma$
fixed, and then set $\sigma'=\sigma$.
The stationary-point condition becomes
\begin{equation}
    \eta\kappa_i
    =2\beta_\eta\left(
        \frac{G_0}{r_i}
        -\frac{G_\eta}{r_i+2\eta}
    \right),
    \qquad i\in\{1,2\}.
    \label{eq:theory-joint-stationarity}
\end{equation}
The same weight $\beta_\eta$ must satisfy this equation
for both coordinates.
For variances bounded above and bounded away from zero,
the expression in parentheses expands uniformly as
\[
    \frac{G_0}{r_i}-\frac{G_\eta}{r_i+2\eta}
    =\eta G_0\left(
        \frac{3}{r_i^2}+\frac{1}{r_1r_2}
    \right)+O(\eta^2).
\]
Dividing the two equations in
\cref{eq:theory-joint-stationarity} eliminates the weight
before taking the small-step limit.
Any finite positive limiting variances must therefore satisfy
\begin{equation}
    64=\frac{r_2(3r_2+r_1)}{r_1(3r_1+r_2)},
    \qquad
    \frac{r_2}{r_1}
    =t:=\frac{21+\sqrt{697}}{2}\approx23.7004.
    \label{eq:theory-joint-stationary-ratio}
\end{equation}
At the Boltzmann target, however,
$r_1=1+2/16$ and $r_2=1+2\cdot4$, giving
$r_2/r_1=8$.
Thus the balance required by the joint loss is
incompatible with the target variances.

\emph{The KL lower bound.}
For this Gaussian family,
\[
    D_{\mathrm{KL}}(\pi_\sigma\|\pi_Q^*)
    =\frac12\sum_{i=1}^2
      \bigl(\kappa_i x_i-1-\log(\kappa_i x_i)\bigr).
\]
This expression diverges if either variance tends
to zero or infinity.
If the KL lower limit is finite, we can therefore
choose a subsequence attaining that limit whose
variances converge to finite positive values.
Those values must satisfy
\cref{eq:theory-joint-stationary-ratio}.
Writing $x=x_1$ and $y=x_2=tx+(t-1)/2$, we obtain
\begin{equation}
    c_{\mathrm{stat}}
    =\min_{x>0}\frac12\left[
        16x-1-\log(16x)
        +\frac{y}{4}-1-\log\left(\frac{y}{4}\right)
    \right]
    \approx0.50735.
    \label{eq:theory-joint-stationary-minimum}
\end{equation}
The function is strictly convex and attains its minimum
at $x=(\sqrt{457}-17)/(64+t)$.
If the KL lower limit is infinite, the bound holds
immediately.
This proves \cref{eq:theory-joint-stationary-gap}.
\end{proof}

\subsection{Numerical Comparison and Weight Selection}
\label{sec:appendix-sequential-joint-toy}
\label{sec:toy-setup}

\begingroup
\newcommand{\sjWeightFigure}{%
  \leavevmode
  \pdfximage{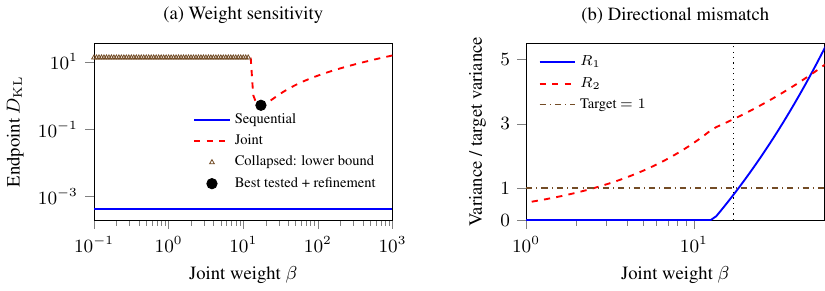}%
  \setbox0=\hbox{\pdfrefximage\pdflastximage}%
  \ht0=139.9264pt\dp0=0pt\wd0=\linewidth
  \box0%
}

\paragraph{Setup.}
We test the Gaussian example using the same initial
actor, $\sigma_0=(0.7,0.5)^\top$, for both methods.
We use $\eta=0.005$ and run 8,000 cycles, corresponding
to evolution time $\tau=40$.
Sequential fits each target exactly, while Joint
numerically minimizes \cref{eq:toy-joint-loss} in each cycle
using damped Newton steps with analytic derivatives.
All action and kernel expectations are evaluated
analytically, so the comparison has no sampling error
in the losses.

We run each of 121 logarithmically spaced weights
in $[10^{-1},10^3]$ from the same initialization and
select the one with the lowest final KL divergence.
We then refine the search around the best tested weight,
obtaining $\beta\approx17.2624$.
The weight remains fixed throughout each run.
KL is used only for selecting and evaluating the weight;
the actor is trained with the regression and MMD losses.

\begin{figure}[!htb]
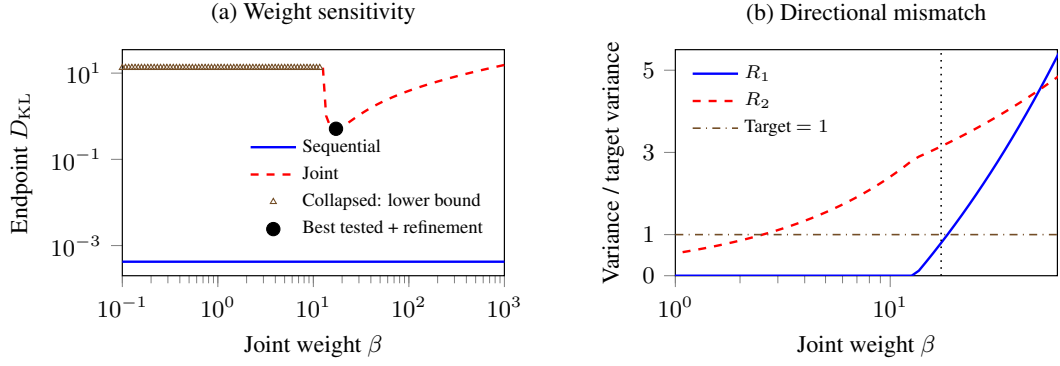

  \centering
  \sjWeightFigure
  \caption{\textbf{Selecting the joint-loss weight in the Gaussian example.}
  (a) Final KL divergence after 8,000 cycles.
  (b) Learned variances divided by target variances;
  a ratio of one indicates an exact match.
  The dot in (a) and dotted line in (b) mark the
  weight selected by the grid search and refinement.
  The 63 runs with
  $\min_i\sigma_i/\sigma_i^*<10^{-6}$ are shown in (a)
  at the KL lower bound $13.3155$ rather than their
  actual KL values.}
  \label{fig:toy-weights}
\end{figure}

\paragraph{Results.}
Figure~\ref{fig:toy-weights} shows the final KL divergence
and {the variance ratio in each dimension},
$R_i=\sigma_i^2/(\sigma_i^*)^2$.
{Since the fitted policies and the Boltzmann target
are zero-mean Gaussians, the variance ratio indicates
how closely each dimension matches the target,
with a ratio of one indicating an exact match.}
Sequential reaches a KL divergence of $0.000422$,
whereas the selected Joint run reaches $0.512165$.
For Joint, the variance ratios are approximately
$(0.80,3.15)$: the first coordinate is too narrow,
while the second is too wide.
{No choice of $\beta$ makes both variance ratios
equal to one at the same time.
This shows that a single $\beta$ cannot always balance
Drift and Spread across dimensions, even when
the action has only two dimensions.}

Figure~\ref{fig:experiment-toy} shows the fitted
distributions and the effect of reducing the step size.
Keeping $\tau=40$ and the selected Joint weight fixed,
we reduce $\eta$ from $0.01$ to $0.000625$.
Sequential's final KL decreases from $0.001788$ to
$6.29\times10^{-6}$, while Joint's remains near
$0.51$.
These results agree with the different small-step
behavior in \cref{prop:theory-joint-counterexample}.

\begin{figure}[!htbp]
  \centering
  \includegraphics[width=\linewidth]
    {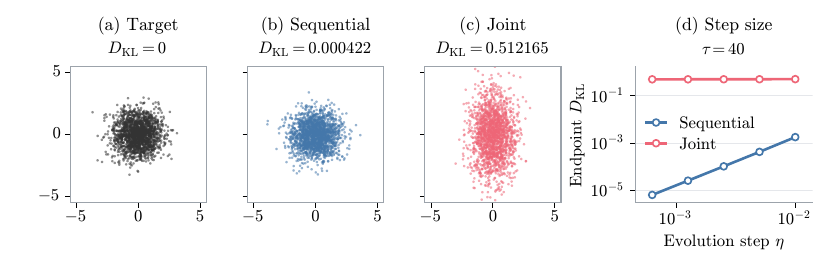}
  \caption{\textbf{Sequential and joint fitting of the Gaussian target.}
  (a) Boltzmann target. (b,c) Samples from the fitted actors.
  Scatter plots use normalized coordinates $a_i/\sigma_i^*$,
  so the target has unit variance in both coordinates.
  (d) Final KL divergence for different step sizes at
  the same evolution time $\tau=40$.
  Joint uses the selected weight $\beta\approx17.2624$.}
  \label{fig:experiment-toy}
\end{figure}

\FloatBarrier
\endgroup

\FloatBarrier
\section{Additional Experimental Results}
\label{sec:appendix-experiments}
\label{sec:appendix-additional-results}

We provide additional experiments to examine how Spread
changes the learned distribution and how its implementation
affects policy learning.
We first extend the ring-target experiment to more
temperatures, then compare
joint-loss weights and frequency scales on MuJoCo.
The corresponding experimental settings are given in
Appendix~\ref{sec:appendix-experimental-details}.

\subsection{Effect of Temperature on the Learned Distribution}
\label{sec:appendix-cfm-comparison}
\label{sec:cfm-toy-eval}

We extend the ring-target experiment in
\cref{sec:experiment-cfm-toy} to
$\lambda\in\{0,0.05,0.1,0.25,0.5,0.75,1\}$.
The critic is unchanged, so increasing $\lambda$
changes the spread of the Boltzmann target without
changing the location of the high-value ring.
This experiment tests whether the actor learns
the effect of the entropy coefficient on its output
distribution.

Figure~\ref{fig:appendix-cfm-ring-extended} shows that
the Boltzmann reference becomes wider around the ring
as $\lambda$ increases.
Drift-only remains concentrated near the ring,
whereas \ours{} learns a broader distribution that
follows this change.
At $\lambda=0$, Spread is omitted, and the reference
is the uniform-circle limit of the Boltzmann targets
as $\lambda$ approaches zero.

\begin{figure}[!htb]
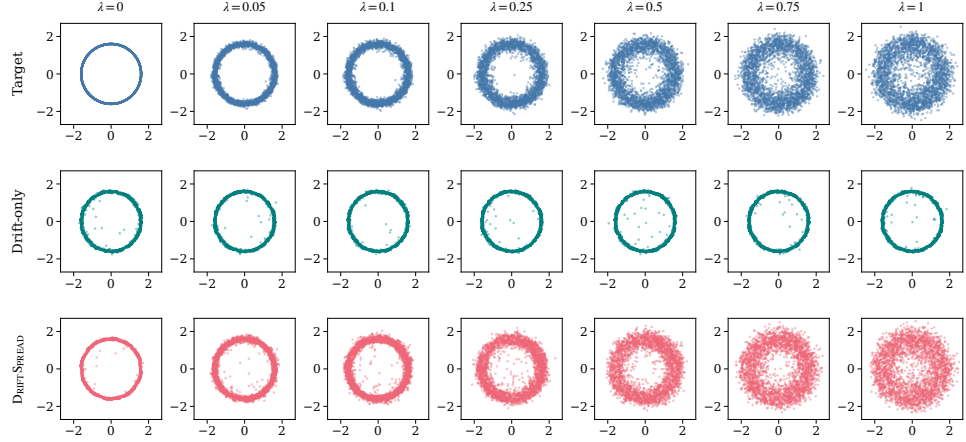

  \centering
  \RingSweepFigure[.96\linewidth]
  \caption{\textbf{Learning the ring-shaped target at different temperatures.}
  Columns vary the entropy coefficient $\lambda$.
  Rows show target samples, Drift-only samples, and
  samples from \ours{}.
  At $\lambda=0$, the reference is uniform on the circle
  and both training procedures omit Spread.}
  \label{fig:appendix-cfm-ring-extended}
  \label{fig:appendix-cfm-ring-sweep}
\end{figure}

\subsection{Sequential Fitting versus Joint-loss Weighting}
\label{sec:experiment-joint-beta}

{We extend the comparison in
Appendix~\ref{sec:appendix-design-analysis} to Humanoid-v3,
using Gaussian Spread noise.}
Joint minimizes
$\mathcal L_{\mathrm{Drift}}
+\beta\mathcal L_{\mathrm{Spread}}^{\mathrm{old}}$
with $\beta\in\{0,0.1,1,10\}$ fixed within each run.
Both targets are constructed from the pre-Drift actor;
the case $\beta=0$ removes the Spread loss.
Sequential instead constructs its Spread target after
the Drift update.
All other training settings are kept the same.

Figure~\ref{fig:experiment-joint-beta} shows that
Sequential achieves higher  mean returns
than Joint for all tested weights.
We can also see that, {compared with the zero-weight control (Drift alone),
adding the Spread loss improves performance for every positive weight considered.}

\begin{figure}[!htbp]
  \centering
  \includegraphics[width=0.60\linewidth]
    {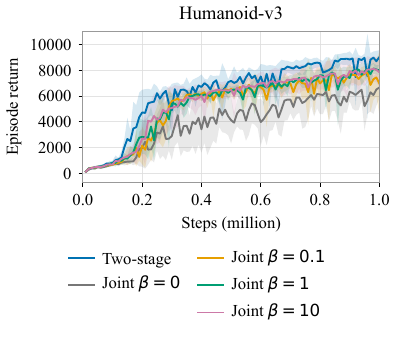}
  \caption{\textbf{Sequential and joint fitting on Humanoid-v3.}
  Joint weights the Spread MMD loss by $\beta$, with
  $\beta=0$ using the Drift loss alone.
  Sequential is labeled \emph{Two-stage} in the figure.
  Curves show unsmoothed means over five seeds, with
  sample standard deviations shaded.}
  \label{fig:experiment-joint-beta}
\end{figure}

\subsection{Sensitivity to the Frequency Scale}
\label{sec:experiment-frequency}
\label{sec:appendix-frequency-sensitivity}

\begin{figure}[!htbp]
  \centering
  \includegraphics[width=0.60\linewidth]
    {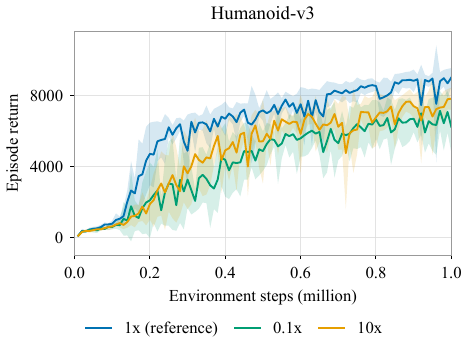}
  \caption{\textbf{Effect of the frequency scale on Humanoid-v3.}
  We multiply the default frequency standard deviation
  by $0.1$, $1$, or $10$, keeping the spreading noise
  and other training settings fixed.
  Curves show unsmoothed means over five seeds, with
  sample standard deviations shaded.}
  \label{fig:experiment-frequency}
\end{figure}

The frequency distribution determines the Gaussian
kernel used to compute the Spread MMD loss.
We test its scale on Humanoid-v3 using
$\sigma_\omega=\alpha_\omega\sigma_{\mathrm G}$,
where $\sigma_{\mathrm G}=(2\lambda\eta d)^{-1/2}$
is the default and
$\alpha_\omega\in\{0.1,1,10\}$.
The Gaussian Spread target and all other training
settings remain fixed.

Figure~\ref{fig:experiment-frequency} shows that
the default scale gives faster learning and higher
late-training mean returns than either change.
Thus the frequency scale affects learning even when
the target distribution is unchanged.

\FloatBarrier

\section{Experimental Details}
\label{sec:appendix-experimental-details}

We first describe the ring-target experiment.
We then give the MuJoCo implementations and training
settings, followed by the evaluation and runtime
measurements and the frequency settings used in the
noise comparison.

\subsection{Ring-target Experiment}
\label{sec:appendix-cfm-toy}

\paragraph{Target distribution.}
We use the two-dimensional critic
\begin{equation}
    Q(a)=-\frac{(\|a\|^2-r_0^2)^2}{4\sigma_Q^2},
    \qquad r_0=1.6,\quad\sigma_Q=0.8.
    \label{eq:cfm-ring-critic}
\end{equation}
For $\lambda>0$, the target is
$\pi_{Q,\lambda}^*(a)\propto\exp(Q(a)/\lambda)$.
To draw reference samples, sample the squared radius
$V$ from $\mathcal N(r_0^2,2\sigma_Q^2\lambda)$
conditioned on $V>0$, and independently sample
$\Theta\sim\operatorname{Uniform}[0,2\pi)$.
Then set $a=\sqrt V(\cos\Theta,\sin\Theta)^\top$.
At $\lambda=0$, use $V=r_0^2$, giving a uniform
distribution on the circle.
Reference samples are used only for evaluation
and visualization.

\paragraph{Actor and initialization.}
The actor is $f_\theta(z)=0.8z+g_\theta(z)$,
where $z\sim\mathcal N(0,I_2)$ and $g_\theta$
is an MLP with two hidden layers of width 64
and $\tanh$ activations.
Its final linear layer is initialized to zero,
so the initial output distribution is
$\mathcal N(0,0.8^2I_2)$.
Drift-only and \ours{} use the same architecture
and initialization.

\paragraph{Training targets.}
For this experiment, the critic gradient flow
can be solved analytically.
We use its exact action movement over duration
$h=0.025$ to construct the Drift target.
Writing $R=r_0^2$ and $v=\|a\|^2$, this movement is
\begin{equation}
    D_h(a)=
    \frac{a}{\sqrt{
        e^{-2Rh/\sigma_Q^2}
        +(v/R)(1-e^{-2Rh/\sigma_Q^2})
    }},
    \qquad D_h(0)=0.
    \label{eq:cfm-ring-drift-target}
\end{equation}
The actor fits the fixed target
$\sg[D_h(f_{\theta_k}(z))]$ through regression
using the same input $z$.
After Drift, we compute the updated actor's empirical
characteristic function and multiply it by
$e^{-\lambda h\|\omega\|^2}$.
We hold this target fixed while optimizing the
Spread MMD loss.

\paragraph{{Toy experiment settings.}}
We train for 400 Drift--Spread cycles.
Each cycle draws 2,048 actor inputs.
For $\lambda>0$, it also draws 128 frequencies from
$\mathcal N(0,\sigma_\omega^2I_2)$ with
$\sigma_\omega=(4\lambda h)^{-1/2}$.
Drift-only uses the same Drift settings and actor
inputs, with Spread omitted.
At $\lambda=0$, \ours{} also omits Spread and
does not draw Spread frequencies.

\paragraph{Displayed samples.}
{Each displayed distribution contains 3,000 samples,
with 600 from each of five independently trained actors.}
Target panels use the same number of points,
and all panels share the same axes.
Every generated sample is obtained with one
actor forward pass.

\subsection{MuJoCo Implementations and Training Settings}
\label{sec:appendix-mujoco-settings}

We use HalfCheetah-v3, Hopper-v3, Walker2d-v3,
Ant-v3, and Humanoid-v3.
Each method is trained for $10^6$ environment
interactions per task using five random seeds.
The compared methods are listed in
\cref{sec:experiment-comparison}.

\paragraph{Baseline implementations.}
We use the official baseline implementations as
references for our JAX implementations.
For methods originally implemented in PyTorch,
we port the implementation to JAX for consistent
runtime comparisons.
We align batch size, network depth and width,
and update-to-data ratio across methods.
Other settings follow the source implementations,
including learning rates, activation functions,
sampling steps, and action-selection rules.
{We adapt DFP from offline-to-online learning to online
learning from scratch using the online RL settings of MFPO.
DPMD and SDAC use MFPO's critic configuration.}

The source repositories for the baseline implementations
are listed in Table~\ref{tab:appendix-baseline-sources}.

\begin{table}[!htbp]
  \centering
  \small
  \caption{Source repositories for the baseline implementations.
  The SAC entry is the implementation referenced by MFPO.
  DPMD and SDAC share the same repository.}
  \label{tab:appendix-baseline-sources}
  \begin{tabular}{@{}lp{0.78\linewidth}@{}}
    \toprule
    Method & Repository \\
    \midrule

    MFPO &
    \href{https://github.com/dongxiaoyi-xyz/MFPO}
         {\nolinkurl{dongxiaoyi-xyz/MFPO}} \\

    DIME &
    \href{https://github.com/ALRhub/DIME}
         {\nolinkurl{ALRhub/DIME}} \\

    FlowRL &
    \href{https://github.com/bytedance/FlowRL}
         {\nolinkurl{bytedance/FlowRL}} \\

    SAC-Flow &
    \href{https://github.com/Elessar123/SAC-FLOW}
         {\nolinkurl{Elessar123/SAC-FLOW}} \\

    MaxEntDP &
    \href{https://github.com/diffusionyes/MaxEntDP}
         {\nolinkurl{diffusionyes/MaxEntDP}} \\

    DACER &
    \href{https://github.com/happy-yan/DACER-Diffusion-with-Online-RL}
         {\nolinkurl{happy-yan/DACER-Diffusion-with-Online-RL}} \\

    QVPO &
    \href{https://github.com/wadx2019/qvpo}
         {\nolinkurl{wadx2019/qvpo}} \\

    TD3 &
    \href{https://github.com/sfujim/TD3}
         {\nolinkurl{sfujim/TD3}} \\

    SAC &
    \href{https://github.com/toshikwa/soft-actor-critic.pytorch}
         {\nolinkurl{toshikwa/soft-actor-critic.pytorch}} \\

    DPMD &
    \href{https://github.com/mahaitongdae/diffusion_policy_online_rl}
         {\nolinkurl{mahaitongdae/diffusion_policy_online_rl}} \\

    SDAC &
    \href{https://github.com/mahaitongdae/diffusion_policy_online_rl}
         {\nolinkurl{mahaitongdae/diffusion_policy_online_rl}} \\

    DFP &
    \href{https://github.com/KAIST-Visual-AI-Group/DFP}
         {\nolinkurl{KAIST-Visual-AI-Group/DFP}} \\

    WPPG &
    \href{https://github.com/Asimov-Chuang/wasserstein-proximal-policy-gradient}
         {\nolinkurl{Asimov-Chuang/wasserstein-proximal-policy-gradient}} \\

    \bottomrule
  \end{tabular}
\end{table}

Table~\ref{tab:baseline-hyperparameters} summarizes the
hyperparameters used for the baseline methods.

\begin{table}[!htbp]
  \centering
  \small
  \setlength{\tabcolsep}{3pt}
  \caption{Hyperparameters of the baseline methods on MuJoCo.
  All methods use a batch size of 256, a discount factor
  $\gamma=0.99$, and an update-to-data ratio of 1.}
  \label{tab:baseline-hyperparameters}

  \begin{tabular}{@{}l*{7}{c}@{}}
    \toprule
    Hyperparameter
      & MFPO & DIME & FlowRL & SAC-Flow
      & MaxEntDP & DACER & QVPO \\
    \midrule
    Target smoothing $\tau$
      & 0.005 & 1 & 0.95 & 1
      & 0.005 & 0.005 & 0.005 \\
    Hidden layers
      & 3 & 3 & 3 & N/A & 3 & 3 & 3 \\
    Hidden nodes per layer
      & 256 & 256 & 256 & N/A & 256 & 256 & 256 \\
    Actor learning rate
      & $3{\times}10^{-4}$ & $3{\times}10^{-4}$
      & $3{\times}10^{-4}$ & $3{\times}10^{-4}$
      & $3{\times}10^{-4}$ & $10^{-4}$
      & $3{\times}10^{-4}$ \\
    Critic learning rate
      & $3{\times}10^{-4}$ & $3{\times}10^{-4}$
      & $3{\times}10^{-4}$ & $10^{-3}$
      & $3{\times}10^{-4}$ & $10^{-4}$
      & $3{\times}10^{-4}$ \\
    Activation
      & GELU & GELU & ELU & GELU & Mish & Mish & Mish \\
    Replay buffer size
      & $10^6$ & $10^6$ & $10^6$ & $10^6$
      & $10^6$ & $10^6$ & $10^6$ \\
    Diffusion / generation steps
      & 2 & 16 & 1 & 4 & 20 & 20 & 20 \\
    Action candidates
      & 10 & N/A & N/A & N/A & 10 & N/A & 32 \\
    \bottomrule
  \end{tabular}

  \par\vspace{8pt}

  \begin{tabular}{@{}l*{6}{c}@{}}
    \toprule
    Hyperparameter
      & TD3 & SAC & DPMD & SDAC & DFP & WPPG \\
    \midrule
    Target smoothing $\tau$
      & 0.005 & 0.005 & 0.005 & 0.005 & 0.005 & 0.005 \\
    Hidden layers
      & 3 & 3 & 3 & 3 & 3 & 3 \\
    Hidden nodes per layer
      & 256 & 256 & 256 & 256 & 256 & 256 \\
    Actor learning rate
      & $3{\times}10^{-4}$ & $3{\times}10^{-4}$
      & $3{\times}10^{-4}$ & $3{\times}10^{-4}$
      & $3{\times}10^{-4}$ & $3{\times}10^{-4}$ \\
    Critic learning rate
      & $3{\times}10^{-4}$ & $3{\times}10^{-4}$
      & $3{\times}10^{-4}$ & $3{\times}10^{-4}$
      & $3{\times}10^{-4}$ & $3{\times}10^{-4}$ \\
    Activation
      & ReLU & ReLU & Mish & Mish & GELU & Tanh \\
    Replay buffer size
      & $10^6$ & $10^6$ & $10^6$ & $10^6$
      & $2{\times}10^6$ & $10^6$ \\
    Diffusion / generation steps
      & N/A & N/A & 20 & 20 & 1 & 1 \\
    Action candidates
      & N/A & N/A & 32 & 32 & 16 & N/A \\
    \bottomrule
  \end{tabular}
\end{table}

\paragraph{Training settings for \ours{}.}
Table~\ref{tab:appendix-shared-hyperparameters}
lists the settings used across all five tasks.
We use the same Drift learning rate on every task
and the Spread learning rates in
Table~\ref{tab:appendix-spread-learning-rates}.
The optimizer learning rates control parameter
updates; the step size $\eta$ controls the action
displacement and spreading-noise variance used
to construct training targets.

\begin{table}[!htbp]
  \centering
  \small
  \setlength{\tabcolsep}{4pt}
  \caption{Training settings for \ours{} on MuJoCo.
  Entries marked $\dagger$ are aligned across all methods.}
  \label{tab:appendix-shared-hyperparameters}
  \label{tab:appendix-drift-cf-hyperparameters}
  \begin{tabular}{@{}p{0.26\linewidth}p{0.17\linewidth}
                    p{0.34\linewidth}p{0.15\linewidth}@{}}
    \toprule
    Parameter & Value & Parameter & Value \\
    \midrule
    Batch size$^\dagger$ & 256
      & Drift samples per state $N_{\mathrm D}$ & 4 \\
    Hidden layers$^\dagger$ & 3
      & Spread samples per state $N_{\mathrm{CF}}$ & 32 \\
    Hidden width$^\dagger$ & 256
      & Frequencies per cycle $M_\omega$ & 64 \\
    Update-to-data ratio$^\dagger$ & 1
      & Spreading noise & Gaussian \\
    Discount $\gamma$ & 0.99
      & Entropy coefficient $\lambda$ & 0.04 \\
    Target update rate & 0.005
      & Step size $\eta$ & 0.1 \\
    Critic learning rate & $3\times10^{-4}$
      & Drift learning rate & $2\times10^{-5}$ \\
    Activation & GELU
      & Evaluation candidates & 10 \\
    Replay buffer size & $10^6$
      & Interactions per seed & $10^6$ \\
    Random seeds & 0--4
      & Frequency distribution & Gaussian \\
    \bottomrule
  \end{tabular}
\end{table}

\begin{table}[!htbp]
  \centering
  \caption{Spread learning rates for \ours{}.
  The same task-specific rates are used in the
  spreading-noise comparison.}
  \label{tab:appendix-spread-learning-rates}
  \begin{tabular}{@{}lc@{}}
    \toprule
    Task & Spread learning rate \\
    \midrule
    HalfCheetah-v3 & $10^{-5}$ \\
    Hopper-v3 & $10^{-5}$ \\
    Walker2d-v3 & $2\times10^{-5}$ \\
    Ant-v3 & $1.5\times10^{-5}$ \\
    Humanoid-v3 & $10^{-5}$ \\
    \bottomrule
  \end{tabular}
\end{table}

\paragraph{Bounded actions.}
The actor produces an unconstrained output
$u=f_\theta(s,z)$, which is mapped to an environment
action by $a=T(u)$.
For each coordinate, $T_i(u_i)=c_i+r_i\tanh u_i$,
where $c_i$ and $r_i$ are the center and half-width
of the action range.
Both actor losses are computed before this mapping.
Drift uses the gradient of
\begin{equation}
    Q^T(s,u)
    =\min_{i\in\{1,2\}}Q_{\phi_i}(s,T(u))
      +\lambda\log J_T(u),
    \qquad J_T(u)=|\det DT(u)|.
    \label{eq:effective-reward}
\end{equation}
The Jacobian term accounts for the change in entropy
under $T$.
Spread acts on the distribution of the updated
actor outputs $u$.
Critic learning uses mapped actions and the
reward-only target in \cref{sec:method-practical}.

\subsection{Evaluation and runtime measurements}
\label{sec:appendix-evaluation-measurements}

\paragraph{Control performance.}
The main MuJoCo learning curves report evaluation
returns averaged over five seeds, with sample standard
deviations shown as shaded bands.
For \ours{} and MFPO, evaluation generates ten
candidate actions and executes the action with the
highest critic value.
Other methods use the action-selection rules of
their reference implementations.

\paragraph{Aggregate normalized return.}
For each task and seed, we average the final ten
evaluation returns.
We divide each result by the corresponding task's
mean return for \ours{}, averaged over its seeds.
We then average the normalized results across
seeds and tasks, giving each task equal weight.
Thus a normalized return of one corresponds to
the \ours{} reference level.
The aggregate error bars are 95\% bootstrap intervals
obtained by resampling seed results while keeping
the task-specific normalization constants fixed.

\paragraph{Training and inference time.}
The methods included in the efficiency comparison
are timed using JAX implementations on the same
NVIDIA RTX 4090 GPU.
We exclude one-time compilation and warm-up from
both training and inference measurements.
Training measurements also exclude evaluation
and experiment-management overhead.

Average training time is the total GPU time for
one seed on all five tasks, each with $10^6$
interactions, divided by five.

Inference time is measured per selected action and
averaged across tasks.
It includes candidate generation, critic scoring,
and final selection when these steps are used.
For \ours{} and MFPO, this includes all ten evaluation
candidates.
Each candidate requires one generation step for
\ours{} and two for MFPO.
The other recorded settings use one step for FlowRL,
four for SAC-Flow, 16 for DIME, and 20 for DACER,
QVPO, and MaxEntDP.

Policy evaluation is performed every 10,000 environment steps using 10 episodes. Each evaluation return is the arithmetic mean
of the undiscounted cumulative rewards over these ten episodes.

\subsection{Frequency Settings for the Noise Comparison}
\label{sec:appendix-frequency-targets}
\label{sec:appendix-frequency-implementation}

The spreading noise determines the target distribution,
while the frequency distribution determines the MMD
kernel used to compare it with the actor.
For every noise choice, we sample
$\omega=\sigma_\omega Z$ with
$Z\sim\mathcal N(0,I_d)$.
The resulting Gaussian kernel is
$h(u,v)=\exp(-\sigma_\omega^2\|u-v\|^2/2)$.
Frequencies are resampled in each actor-update cycle;
their distribution is fixed throughout a run.

For the default Gaussian noise, we use
$\sigma_{\mathrm G}=(2\lambda\eta d)^{-1/2}$,
as described in Appendix~\ref{sec:appendix-frequency}.
Let $m_\eta$ denote the characteristic function of
the spreading noise.
At the default scale, the Gaussian noise factor
has average value
\begin{equation}
    \mu_*=
    \E_Z\left[
        e^{-\lambda\eta\sigma_{\mathrm G}^2\|Z\|^2}
    \right]
    =(1+1/d)^{-d/2}.
    \label{eq:frequency-reference-mean}
\end{equation}
For each alternative, we choose its frequency scale
so that its noise factor has the same average value:
\begin{equation}
    F(\sigma_\omega)=\mu_*,
    \qquad
    F(\sigma)=\E_Z[m_\eta(\sigma Z)].
    \label{eq:frequency-matched-mean}
\end{equation}
This provides a common rule for setting the frequency
scale across noise choices.

To compute the scale, let $\epsilon_\eta$ be a noise
sample independent of $Z$.
Averaging over the Gaussian frequency first gives
\begin{equation}
    F(\sigma)
    =\E_{\epsilon_\eta}
      \left[e^{-\sigma^2\|\epsilon_\eta\|^2/2}\right].
    \label{eq:frequency-increment-duality}
\end{equation}
For the noise distributions used here, this function
decreases continuously from one to zero.
The required expectations have closed forms or reduce
to one-dimensional integrals.
We evaluate them numerically and solve
\cref{eq:frequency-matched-mean} by bisection.

For the Gamma-mixture noise used in the comparison,
the shape parameter is fixed at {$\kappa=4$}.

\end{document}